\documentclass{article}

\PassOptionsToPackage{numbers,sort&compress}{natbib}
\usepackage[preprint]{neurips_2026}

\usepackage[utf8]{inputenc}
\usepackage[T1]{fontenc}
\usepackage{fontawesome}
\usepackage{hyperref}
\usepackage{url}
\usepackage{amsmath,amssymb,amsfonts}
\usepackage{booktabs}
\usepackage{nicefrac}
\usepackage{graphicx}
\usepackage[table]{xcolor}
\usepackage{cleveref}
\usepackage{subcaption}
\usepackage{multirow}
\usepackage{arydshln}
\usepackage{enumitem}
\usepackage{microtype}
\usepackage{mathtools}
\usepackage{amsthm}
\usepackage{dsfont}
\usepackage{tikz}
\newtheorem{proposition}{Proposition}
\newtheorem{corollary}[proposition]{Corollary}

\newtheorem{remark}{Remark}

\newif\ifdraftmode \draftmodefalse
\ifdraftmode
  
  \newcommand{\todo}[1]{{\color{orange}\textbf{[TODO: #1]}}}
  
\else
  
  \newcommand{\todo}[1]{}
  
\fi

\newcommand{\R}{\mathbb{R}}
\newcommand{\E}{\mathbb{E}}
\newcommand{\x}{\mathbf{x}}
\newcommand{\h}{\mathbf{h}}

\newcommand{\loss}{\mathcal{L}}
\newcommand{\encoder}{f_\theta}
\newcommand{\predictor}{g_\phi}
\newcommand{\target}{\bar{f}_\theta}
\newcommand{\pimarg}{\pi_e}

\title{HEPA: A Self-Supervised Horizon-Conditioned\\Event Predictive Architecture for Time Series}
\author{%
  \textbf{Jonas Petersen}$^{1,2}$\qquad
  Gian-Alessandro Lombardi$^{2}$\qquad
  Riccardo Maggioni$^{2}$\\[2pt]
  \textbf{Camilla Mazzoleni}$^{2}$\qquad
  \textbf{Federico Martelli}$^{1,2}$\qquad
  \textbf{Philipp Petersen}$^{3}$\\[4pt]
  $^{1}$ETH Zurich\quad
  $^{2}$Forgis\quad
  $^{3}$University of Vienna\\[2pt]
  Correspondence to \texttt{jep79@cantab.ac.uk}
}

\begin{document}
\maketitle
\vskip -0.5em
\centerline{\small\faGithub\ \href{https://github.com/Forgis-Labs/HEPA}{\texttt{github.com/Forgis-Labs/HEPA}}}
\vskip 0.5em

\begin{abstract}
  Critical events in multivariate time series, from turbine failures to cardiac arrhythmias, demand accurate prediction, yet labeled data is scarce because such events are rare and costly to annotate. We introduce HEPA (Horizon-conditioned Event Predictive Architecture), built on two key principles. First, a causal Transformer encoder is pretrained via a Joint-Embedding Predictive Architecture (JEPA): a horizon-conditioned predictor learns to forecast future \emph{representations} rather than future values, forcing the encoder to capture predictable temporal dynamics from unlabeled data alone. Second, we freeze the encoder and \emph{finetune only the predictor} toward the target event, producing a monotonic survival cumulative distribution function (CDF) over horizons. With fixed architecture and optimiser hyperparameters across all benchmarks, HEPA handles water contamination, cyberattack detection, volatility regimes, and eight further event types across 11 domains, exceeding leading time-series architectures including PatchTST, iTransformer, MAE, and Chronos-2 on at least 10 of 14 benchmarks, with an order of magnitude fewer tuned parameters and, on lifecycle datasets, an order of magnitude less labeled data.
\end{abstract}

\section{Introduction}
\label{sec:intro}

\begin{figure}[t]
  \centering
  \includegraphics[width=\textwidth]{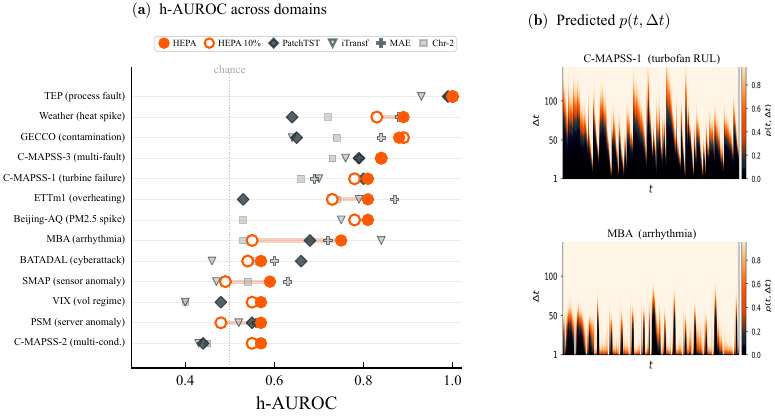}
  \caption{\textbf{One label-efficient architecture, domain- and event-agnostic.} (a)~h-AUROC ($\uparrow$; horizon-averaged AUROC) across 14 benchmarks in 11 domains. HEPA wins on 10 out of 14 at full labels; at 10\% labels (open circles) it retains $\geq$92\% of full-label performance on lifecycle datasets. (b)~Predicted probability surfaces $p(t, \Delta t)$ for turbofan degradation (top) and cardiac arrhythmia (bottom).}
  \label{fig:hero}
\end{figure}

A turbine blade cracks after 12{,}000 flight hours. A bearing degrades over weeks of vibration data. A satellite sensor drifts silently for 48 hours before triggering a cascade. These events are rare in operational data, yet they follow partially predictable precursor dynamics~\citep{scheffer2009earlywarning}: temperatures rise gradually before overheating, vibration amplitudes grow before mechanical failure, and sensor readings deviate systematically before spacecraft faults. A range of machine-learning methods attempt to predict such events from multivariate sensor streams. Remaining-useful-life (RUL) models~\citep{fan2024star} estimate how long until a machine fails; anomaly detectors~\citep{xu2022anomalytransformer,mtsjepa2026} flag when sensor readings look abnormal. Although general-purpose architectures exist for both, the two communities develop separate benchmarks, metrics, and evaluation protocols: RUL models never see anomaly benchmarks; anomaly detectors never forecast time-to-failure. Yet all these tasks share the same structure: given observations up to time $t$, estimate the probability $P(\text{event within } \Delta t)$ for each prediction horizon $\Delta t$.

This structural uniformity suggests a separation of concerns. The \emph{encoder} learns temporal dynamics from unlabeled data without knowing which event matters downstream. The \emph{predictor}, finetuned with a small number of event labels, specialises the learned dynamics to whichever event is relevant. The key design choice is what the encoder should forecast during pretraining. Value-forecasting approaches, whether supervised~\citep{nie2023patchtst} or pretrained on large corpora~\citep{ansari2024chronos,das2024timesfm}, shape representations around all variation in the signal, including noise irrelevant to the downstream event. The Joint-Embedding Predictive Architecture (JEPA)~\citep{assran2023ijepa} offers an alternative: by forecasting future \emph{representations} rather than future values, the encoder learns a latent space that retains what is predictable about the future and discards what is not.

We apply this principle to time series as HEPA (Horizon-conditioned Event Predictive Architecture). A causal Transformer encodes observations up to time $t$; a horizon-conditioned predictor maps the encoding and a horizon $\Delta t$ to a predicted future representation, forcing the encoder to internalise dynamics at multiple timescales (\cref{fig:hero}). After self-supervised pretraining, the standard JEPA recipe discards the predictor and trains a linear probe on the frozen encoder. We instead retain the predictor: freeze the encoder but finetune the predictor alongside a lightweight event head that outputs a discrete-time survival CDF, ensuring that the predicted event probability never decreases as the horizon grows. This ``predictor finetuning'' recipe tunes only 198K parameters, roughly $11{\times}$ fewer than end-to-end training, yet is more expressive than a linear probe because the predictor reshapes its horizon-conditioned outputs to align with the downstream event.

Our contributions are:
\begin{enumerate}[nosep]
  \item \textbf{One architecture, any event, any domain.} A single 2.16\,M-parameter architecture with fixed hyperparameters, evaluated on 14 benchmarks across 11 domains via a unified probability surface $p(t, \Delta t)$. HEPA wins on 10 out of 14 benchmarks while tuning $11{\times}$ fewer parameters than PatchTST.

  \item \textbf{Predictor finetuning as the downstream recipe.} Freezing the encoder and finetuning only the predictor and event head tunes $11{\times}$ fewer parameters than end-to-end training. On the C-MAPSS benchmark~\citep{saxena2008cmapss}, where degradation unfolds over hundreds of cycles, HEPA retains 92\% of full-label h-AUROC at just 2\% of labels. An information-theoretic bound (\cref{prop:retention}) formalises when and why this works, and the bound's key prediction, that lower pretraining loss implies stronger downstream performance, is consistent with the empirical trend across 14 datasets (\cref{fig:theory_validation}).
\end{enumerate}

\section{Related Work}
\label{sec:related}

\paragraph{Self-supervised learning for time series.}
Self-supervised learning (SSL) for time-series representation learning falls into three families. Contrastive methods, including TS2Vec~\citep{yue2022ts2vec}, TNC~\citep{tonekaboni2021tnc}, TimesURL~\citep{liu2024timesurl}, CPC~\citep{oord2018cpc}, and CoST~\citep{woo2022cost}, learn representations by contrasting positive and negative pairs. Masked reconstruction approaches such as PatchTST~\citep{nie2023patchtst}, SimMTM~\citep{dong2023simmtm}, and TimesNet~\citep{wu2023timesnet} recover masked patches in input space. JEPA~\citep{assran2023ijepa,bardes2024vjepa} takes a different path: predicting future \emph{representations} rather than reconstructing inputs, avoiding tying the latent space to value-level fidelity. For time series, TS-JEPA~\citep{tsjepa2024} applies temporal masking for classification, and MTS-JEPA~\citep{mtsjepa2026} adds codebook regularisation for anomaly detection. All these methods discard their pretraining head at inference and probe only the encoder. HEPA instead retains the predictor and finetunes it toward the downstream event, treating the predictor as a learnable bridge between frozen representations and event probabilities. The collapse-prevention mechanism follows the LeJEPA / SIGReg line~\citep{balestriero2025lejepa} rather than the EMA schedule of I-JEPA.

\paragraph{Foundation models for time series.}
Chronos-2~\citep{ansari2024chronos}, TFM-2.5~\citep{das2024timesfm}, MOMENT~\citep{goswami2024moment}, Moirai~\citep{woo2024unified}, and UniTS~\citep{gao2024units} pretrain on large-scale corpora for generic value forecasting. Generative pretraining~\citep{liu2024timer} and LLM repurposing~\citep{zhou2023gpt4ts} offer alternative transfer strategies. These approaches target future channel values; HEPA targets event probabilities. See also concurrent work on industrial pretraining corpora~\citep{othman2025factorynet,merzouki2025factorybench}. The encoder is mid-scale and pretrained per-dataset; what transfers across domains is the \emph{recipe} (architecture + predictor finetuning), not the weights. We benchmark HEPA against four of these foundation models, using identical downstream heads to isolate encoder quality (\cref{sec:experiments,app:chronos_full,app:extra_baselines}).

\paragraph{Prognostics, anomaly prediction, and survival modelling.}
C-MAPSS~\citep{saxena2008cmapss} is the standard remaining-useful-life (RUL) benchmark, where the supervised state of the art is STAR~\citep{fan2024star} (root mean square error, RMSE, 10.61). Self-supervised approaches to RUL prediction remain limited~\citep{ding2022ssl_rul,wang2024masked_rul}. Anomaly detection methods such as Anomaly Transformer~\citep{xu2022anomalytransformer}, DCdetector~\citep{yang2023dcdetector}, and TranAD~\citep{tuli2022tranad} report point-adjusted F1, a metric shown to inflate scores dramatically by crediting entire segments from a single detection~\citep{kim2022rigorous,schmidl2022tsadeval}. These domain-specific metrics are incomparable across tasks. HEPA's downstream parameterisation builds on discrete-time survival models~\citep{lee2018deephit,gensheimer2019nnetsurvival}, which decompose event probability into per-interval hazards composed into a survival CDF; we adapt this to a multi-horizon event prediction setting. We unify evaluation through h-AUROC, the mean of per-horizon AUROC values computed over the probability surface, which is threshold-free and robust to class imbalance (\cref{sec:metrics}). Domain-specific metrics are reported as lossy projections of the same surface for comparability with published baselines.

\section{Method}
\label{sec:method}

\subsection{Architecture and Pretraining}
\label{sec:architecture}

\begin{figure}[t]
  \centering
  \includegraphics[width=\textwidth]{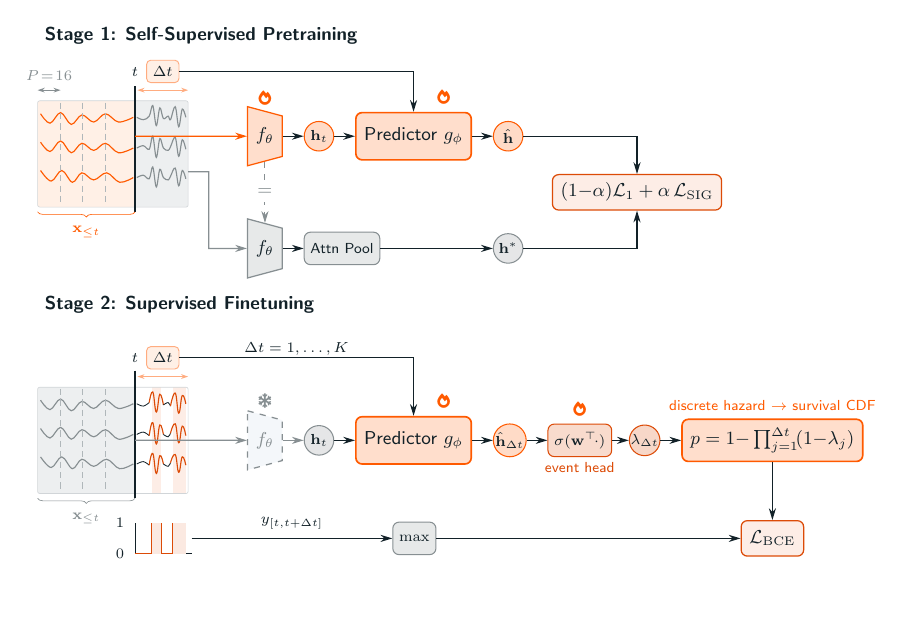}
  \caption{\textbf{HEPA architecture.} Both stages sweep over all $(t, \Delta t)$ pairs per episode. \emph{Stage~1:} The causal encoder $\encoder$ maps $\x_{\leq t}$ to $\h_t$; the predictor $\predictor(\h_t, \Delta t)$ predicts future representations via a self-supervised JEPA objective. \emph{Stage~2:} Encoder frozen; the predictor produces $K$ horizon-specific hazard rates $\lambda_{\Delta t}$ composed into a survival CDF (cumulative distribution function) $p(t, \Delta t)$.}
  \label{fig:architecture}
\end{figure}

HEPA consists of three components that interact across two phases (\cref{fig:architecture}). The \textbf{context encoder} $\encoder$ is a causal Transformer ($d{=}256$, 2 layers, 4 heads) that maps observations $\x_{\leq t}$, tokenised into non-overlapping patches of size $P{=}16$ (following PatchTST~\citep{nie2023patchtst}) with per-context instance normalisation~\citep{kim2022revin} and sinusoidal positional encodings, to a summary embedding $\h_t = \encoder(\x_{\leq t}) \in \R^d$. The \textbf{predictor} $\predictor$ is a 2-layer multilayer perceptron (MLP) that takes the encoder output $\h_t$ together with a prediction horizon $\Delta t$ and produces a predicted embedding of the future interval:
\begin{equation}
  \hat{\h}_{(t,t+\Delta t]} = \predictor(\h_t, \Delta t).
  \label{eq:predictor}
\end{equation}
During pretraining, $\Delta t$ is sampled from a log-uniform distribution over $[1, \Delta t_\text{max}]$, forcing the encoder to internalise dynamics at multiple timescales. The same encoder $\encoder$, applied bidirectionally to $\x_{(t, t+\Delta t]}$ with attention pooling, produces the \textbf{target representation} $\h^*_{(t,t+\Delta t]} \in \R^d$. Both encoders are trained jointly via the optimizer; a SIGReg (Sketched Isotropic Gaussian Regularisation) term $\loss_\mathrm{SIG}$~\citep{balestriero2025lejepa} on the predictor output prevents representation collapse, replacing the exponential moving average (EMA) momentum schedule used in standard JEPA (\cref{app:sigreg_ablation}). SIGReg constrains the predicted representations toward an isotropic Gaussian, which~\citet{balestriero2025lejepa} prove is the optimal embedding distribution for minimising downstream prediction risk in joint-embedding architectures; this eliminates collapse without ad-hoc heuristics. A single mixing weight $\alpha{=}0.1$ controls its contribution to the total loss (\cref{app:sigreg_ablation}).

\paragraph{Relation to canonical JEPA.} HEPA differs from BYOL/I-JEPA/V-JEPA-style joint-embedding predictive architectures in two ways: (a) the target encoder is a weight-shared copy of $\encoder$ rather than an EMA copy or a stop-gradient branch, and (b) collapse is prevented by SIGReg (isotropic Gaussian constraint on the predictor output) rather than by the online/target asymmetry. Trivial collapse $\hat H = H^*=$const is prevented jointly by SIGReg \emph{and} by the asymmetric inputs ($\x_{\leq t}$ for the online branch vs.\ $\x_{(t, t+\Delta t]}$ for the target branch): the predictor never sees the future window directly. This puts HEPA closer to LeJEPA / SIGReg variants~\citep{balestriero2025lejepa} than to the original I-JEPA recipe.

The pretraining loss combines an L1 prediction objective (chosen over L2 because L1 distributes gradient magnitude equally across samples, avoiding domination by outlier predictions) with the SIGReg regulariser:
\begin{equation}
  \loss = (1 - \alpha)\,\|\hat{\h} - \h^*\|_1 + \alpha\,\loss_{\mathrm{SIG}},
  \label{eq:pretrain_loss}
\end{equation}
where $\alpha$ balances the two terms. Because the target encoder shares weights with the online encoder, no stop-gradient is needed; both receive gradients through the optimizer. No labels are used. Pretraining takes under one minute per dataset on a single A10G GPU, with the full 14-dataset, 5-seed sweep completing in under two hours. Per-dataset preprocessing details are in \cref{app:preprocessing}.

\subsection{Downstream: Predictor Finetuning}
\label{sec:downstream}

After pretraining, we freeze the encoder $\encoder$ and finetune only the predictor $\predictor$ together with a lightweight linear event head. This ``predictor finetuning'' (pred-FT) recipe tunes 198K parameters, compared to 2.16M for end-to-end training and 513 for a frozen linear probe. Finetuning reshapes the predictor's per-horizon outputs to separate event-relevant from event-irrelevant dynamics, making it more expressive than a linear probe, while the frozen encoder supplies the pretrained dynamical knowledge that makes few labels sufficient. End-to-end finetuning achieves equivalent h-AUROC at full labels (\cref{tab:finetune_ablation}); pred-FT's advantage is computational efficiency and robustness under label scarcity (\cref{sec:label_efficiency}).

The predictor is run at each of $K$ discrete horizons $\Delta t = 1, \ldots, K$ (unit steps; $K{=}150$ for C-MAPSS/TEP, $K{=}200$ otherwise). A shared linear head maps each predicted representation to a per-interval \emph{conditional hazard}:
\begin{equation}
  \lambda_{\Delta t}(t) \;=\; \sigma\!\bigl(\mathbf{w}^\top \hat{\h}_{(t, t+\Delta t]} + b\bigr) \;\in\; (0, 1),
  \label{eq:hazard}
\end{equation}
where $\sigma$ is the sigmoid function and $\lambda_{\Delta t}(t)$ approximates $P(\text{event in } (\Delta t{-}1, \Delta t] \mid T^* > \Delta t{-}1, \mathbf{x}_{\leq t})$, with $T^*$ denoting the time to the first event after $t$. The event probability surface is then parameterised as a discrete-time survival CDF~\citep{lee2018deephit,gensheimer2019nnetsurvival}:
\begin{equation}
  p(t, \Delta t) \;=\; 1 - \prod_{j=1}^{\Delta t} (1 - \lambda_j(t)).
  \label{eq:event_prob}
\end{equation}
Because each factor $(1 - \lambda_j) \in (0,1)$, the survival product is non-increasing in $\Delta t$, so $p(t, \Delta t)$ increases monotonically with the prediction horizon by construction. No distributional assumptions are required: each $\lambda_{\Delta t}$ is a free function of $\h_t$ via the predictor network. The finetuning loss sums positive-weighted binary cross-entropy (BCE) over horizons:
\begin{equation}
  \loss_\text{FT} = \sum_{\Delta t=1}^{K} w^+ \cdot \text{BCE}\bigl(p(t, \Delta t),\; y(t, \Delta t)\bigr),
  \label{eq:ft_loss}
\end{equation}
where $y(t, \Delta t) = \mathds{1}[\text{event in } (t, t{+}\Delta t]]$ and $w^+ = N_{\text{neg}} / N_{\text{pos}}$ compensates for class imbalance.\footnote{We apply BCE to the cumulative event probability $p(t, \Delta t)$ rather than to the per-step hazards $\lambda_j(t)$ against per-step indicators (the standard discrete-survival likelihood, e.g.\ nnet-survival~\citep{gensheimer2019nnetsurvival}). This is a deliberate design choice: BCE on the cumulative surface acts as a smoothing regulariser across horizons (each hazard $\lambda_j$ contributes to BCE for every $\Delta t \geq j$), which empirically improves h-AUROC under our positive-weighted regime but distorts the probability scale (\cref{app:calibration}).}

\subsection{Theoretical Analysis}
\label{sec:theory}

Predictor finetuning rests on a premise: the pretrained encoder retains enough event-relevant information that a small downstream head can extract it. We formalise when this holds and connect the bound to experiments.

Let $X_{\leq t}$ denote observations up to time $t$, and let $E_{t+\Delta t} \in \{0,1\}$ be a binary indicator that equals 1 if an event occurs in the interval $(t, t{+}\Delta t]$ and 0 otherwise. The encoder produces $H_t = \encoder(X_{\leq t}) \in \R^d$; the target encoder produces $H^* = \target(X_{(t,t+\Delta t]}) \in \R^d$ from the future interval; and the predictor produces $\hat{H} = \predictor(H_t, \Delta t)$. We define the event posterior $\eta(h) \coloneqq P(E_{t+\Delta t}{=}1 \mid H^*{=}h)$ and the marginal event rate $\pimarg \coloneqq P(E_{t+\Delta t}{=}1)$, using $\pimarg$ to distinguish it from the probability surface $p(t, \Delta t)$.

\begin{proposition}[Event-Information Retention]
  \label{prop:retention}
  Suppose \textup{(A1)} the event $E_{t+\Delta t}$ is conditionally independent of $X_{\leq t}$ given $H^*$,
  \textup{(A2)} the pretraining loss satisfies $\E[\|\hat{H} - H^*\|_2^2] \leq \varepsilon$,
  \textup{(A3)} the event posterior $\eta(h)$ is $L$-Lipschitz,
  and \textup{(A4)} the posterior is bounded: $\eta(H^*) \in [\underline{\eta}, \overline{\eta}] \subset (0,1)$ a.s.
  Then
  \begin{equation}
    \label{eq:main_bound}
    I(H_t;\, E_{t+\Delta t}) \;\geq\; I(H^*;\, E_{t+\Delta t}) \;-\; C_\eta \, L^2 \, \varepsilon\;,
  \end{equation}
  where $C_\eta = (2\,\underline{\eta}\,(1{-}\overline{\eta}))^{-1}$ and $I(\cdot;\cdot)$ denotes mutual information.
\end{proposition}

The proof proceeds in three steps (full details in \cref{app:theory}). First, because $\hat{H}$ is a deterministic function of $H_t$, the data processing inequality gives $I(H_t; E) \geq I(\hat{H}; E)$. Second, a Jensen-gap argument on the convex KL divergence, combined with the Lipschitz condition and prediction error bound, yields $I(H^*; E) - I(\hat{H}; E) \leq C_\eta L^2 \varepsilon$. Combining these two inequalities produces the result.

The bound makes a falsifiable prediction: as pretraining proceeds and $\varepsilon$ shrinks, downstream h-AUROC should rise. The bound's constants $L$ (Lipschitz of $\eta$), $C_\eta$ (posterior bound), and the target sufficiency $I(H^\star; E_{t+\Delta t})$ are functions of the data-generating process: they vary across datasets but are held fixed within a dataset. The bound is therefore directly testable only \emph{within} a dataset, by varying $\varepsilon$ alone. We do this on three contrasting domains, turbofan lifecycle (C-MAPSS-3), cardiac arrhythmia (MBA), and spacecraft telemetry anomalies (SMAP), by snapshotting the encoder during pretraining at epochs $\{1, 3, 8, 25\}$ plus the converged best, and at each snapshot running the standard predictor finetuning recipe to obtain h-AUROC on the held-out test split (3 seeds per dataset). The bound's monotone prediction holds across all three: pooled Spearman $\rho(\varepsilon, \text{h-AUROC}) = -0.67$ ($p{=}0.017$, $n{=}12$) on C-MAPSS-3, $\rho{=}{-}0.64$ ($p{=}0.026$, $n{=}12$) on MBA, and $\rho{=}{-}0.49$ ($p{=}0.13$, $n{=}11$) on SMAP. SMAP shows the largest visible h-AUROC range (0.40 at $\varepsilon{=}0.033$ rising to 0.65 at $\varepsilon{=}0.026$). The converged-best snapshot regresses slightly relative to epoch~25 on all three datasets, consistent with mild over-pretraining at fixed labels. C-MAPSS-1 (the original lifecycle benchmark, $\rho{=}{-}0.87$, $p{<}0.001$) gives an even stronger signal and is reported in \cref{app:theory}. \Cref{cor:precursor} predicts a fourth regime where the bound becomes vacuous: on short-window anomaly benchmarks like GECCO we observe a within-dataset $\rho{=}{+}0.14$ ($p{=}0.67$) with finetuning instability across early snapshots, exactly as expected when extended precursors are weak (also in \cref{app:theory}).

\begin{figure}[t]
  \centering
  \includegraphics[width=0.49\columnwidth]{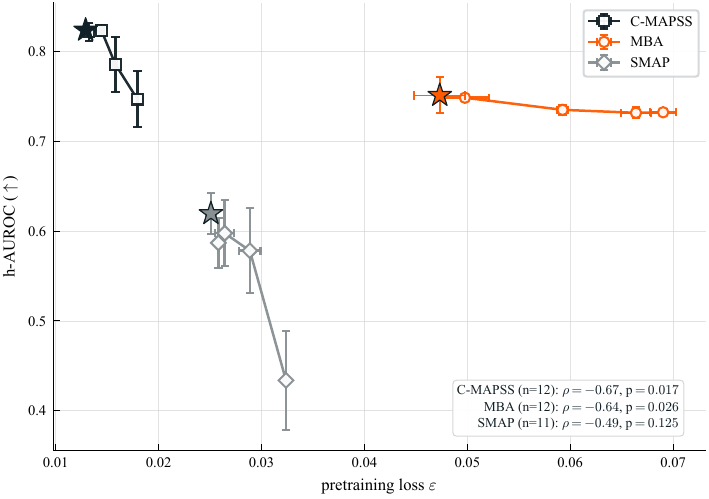}\hfill
  \includegraphics[width=0.49\columnwidth]{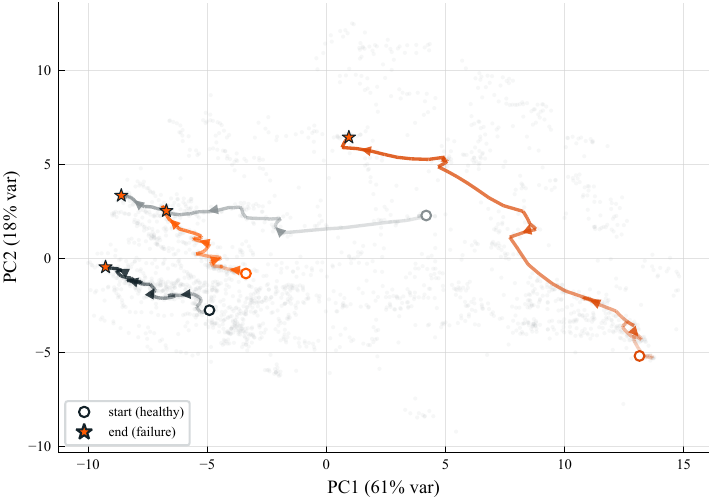}
  \caption{\textbf{Self-supervised pretraining learns task-relevant structure.} (a)~Pretraining loss $\varepsilon$ vs.\ downstream h-AUROC ($\uparrow$) at fixed checkpoints across three domains (C-MAPSS-3: $\rho{=}{-}0.67$; MBA: $\rho{=}{-}0.64$; SMAP: $\rho{=}{-}0.49$; 3 seeds, error bars $\pm 1$\,std). Within a dataset, $L$, $C_\eta$, and $I(H^\star; E_{t+\Delta t})$ are constant, so the bound's monotone prediction is directly testable. $\bigstar$~marks the converged-best snapshot; $\varepsilon$ scales differ across datasets so curves cannot be compared horizontally. (b)~Principal component analysis (PCA) of pretrained C-MAPSS-1 representations for four test engines. Open circles: first observation (healthy); stars: last observation (near failure). PC1 captures 61\% of variance; the encoder organises representations into a smooth degradation manifold without any labels.}
  \label{fig:pretrain_validation}\label{fig:theory_validation}\label{fig:pca}
\end{figure}

A cross-dataset scatter, by contrast, does \emph{not} validate the bound: pooling the converged $\varepsilon$ across the 14 Table~\ref{tab:benchmark} datasets gives Pearson $r{=}{-}0.05$ ($p{=}0.90$), because $L$, $C_\eta$, and the absolute scale of the target representation differ by dataset, dominating any signal from $\varepsilon$ alone; the same incommensurability \cref{fig:theory_validation} makes visible (C-MAPSS-3 clusters around $\varepsilon{\sim}0.015$, SMAP around $0.027$, MBA around $0.06$). This does not contradict the bound; it shows that comparing $\varepsilon$ across datasets compares incommensurable quantities, motivating the within-dataset protocol above. Two further caveats remain. The constants $L$ and $C_\eta$ are not estimated directly; \cref{fig:theory_validation} validates only the monotonic relationship, not the full quantitative bound. And A1 (target sufficiency) may fail when event precursors span intervals longer than the target window; when A1 is violated, the bound becomes loose in a \emph{favourable} direction (see \cref{app:theory} for assumption-by-assumption failure modes).

\begin{corollary}[Precursor necessity]
  \label{cor:precursor}
  The bound is non-vacuous if and only if the future interval contains event precursors that the target encoder captures ($I(H^*; E_{t+\Delta t}) > 0$) and the predictor approximates the target well enough ($\varepsilon < I(H^*; E_{t+\Delta t}) / (C_\eta L^2)$).
\end{corollary}

This corollary explains both HEPA's successes and its failures. On C-MAPSS, degradation unfolds over hundreds of cycles, so $I(H^*; E_{t+\Delta t})$ is large and pretraining drives $\varepsilon$ small, yielding h-AUROC $\geq 0.81$. On datasets without extended precursors, the bound is vacuous regardless of pretraining quality.

\section{Evaluation Framework}
\label{sec:metrics}

The model outputs a probability surface $p(t, \Delta t)$ (\cref{eq:event_prob}) for each observation time $t$ and prediction horizon $\Delta t$. This surface is the complete prediction; every metric is computed deterministically from it (\cref{fig:auroc_curve}), enabling direct comparison with published baselines without retraining.

\begin{figure}[t]
  \centering
  \includegraphics[width=\columnwidth]{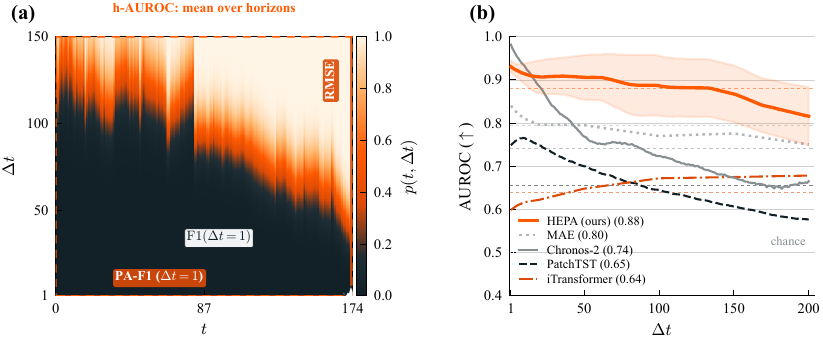}
  \caption{\textbf{Evaluation framework.} (a)~The probability surface $p(t, \Delta t)$ on a representative C-MAPSS-1 engine (lifetime 174 cycles) unifies all event-prediction metrics as lossy projections. The colour scale matches Fig.~\ref{fig:hero}b. RMSE requires converting the survival curve to a point estimate $\hat{\tau} = \sum_{\Delta t} \Delta t \cdot P(\text{event at } \Delta t)$; this projection is sensitive to calibration (\cref{app:legacy_metrics}). PA-F1 thresholds $p(t, 1)$ at the smallest horizon and credits entire anomaly segments from a single detection (inflated~\citep{kim2022rigorous}). F1 collapses to a single $(t, \Delta t)$ cell. h-AUROC averages AUROC over all horizons, using the full surface. (b)~Per-horizon AUROC on GECCO ($K{=}200$ for HEPA / PatchTST / Chronos-2; sparse $K{=}8$ for iTransformer / MAE following the v34 protocol). Mean h-AUROC ($\uparrow$) per method shown in the legend; dashed lines mark the per-method mean. HEPA holds AUROC ${\geq} 0.82$ across the full horizon range while value-level baselines decay sharply.}
  \label{fig:auroc_curve}
\end{figure}

As a cross-domain metric, we use \textbf{h-AUROC}: the mean of per-horizon AUROC values pooled over $(t, \Delta t)$ cells. Per-horizon prevalence varies wildly across datasets, and even within a single surface: on C-MAPSS-1, the event ``failure within $\Delta t$ steps'' has prevalence 0.5\% at $\Delta t{=}1$ and 96\% at $\Delta t{=}150$, a ${\sim}200{\times}$ range. Pooled area under the precision-recall curve (AUPRC) over all $(t, \Delta t)$ cells inherits a 0.957 baseline on C-MAPSS-1, because a model predicting only per-horizon prevalence already scores there. h-AUROC solves this by decomposing the surface into independent per-horizon binary classification problems, each with a universal 0.5 baseline that does not depend on prevalence. The uniform average treats all horizons equally; in practice, specific horizons matter more (long-range for turbine maintenance, short-range for arrhythmia). We use the uniform average for cross-domain comparability; the full surface is always stored for application-specific weighting. Domain-specific metrics (RMSE for remaining-useful-life, PA-F1 for anomaly detection) are derived as projections of the same surface for comparability with published baselines (\cref{app:legacy_metrics}). All numbers are reported as mean $\pm$ std across 5 seeds (HEPA, PatchTST, iTransformer, MAE) or 3 seeds (Chronos-2).

\section{Experiments}
\label{sec:experiments}

\subsection{Setup}
\label{sec:setup}

We pretrain a separate HEPA encoder per dataset from unlabeled training data. Architecture and hyperparameters are identical across all domains; only the input projection (sensor count $S$) changes. All comparison methods share the same 198K-param downstream MLP head, positive-weighted BCE loss, and evaluation protocol; only the frozen encoder differs. Dense unit-step horizons are used throughout: $K{=}150$ for C-MAPSS and TEP, $K{=}200$ for all others. The dataset overview (14 datasets, 11 domains) is in \cref{tab:tokenization}.

\subsection{Main Results}
\label{sec:benchmark}

\begin{table*}[t]
  \centering
  \caption{\textbf{Main results} (mean\,{\scriptsize$\pm$}\,std; 5 seeds for HEPA, PatchTST, iTransformer, MAE; 3 seeds for Chronos-2). All methods use matched-capacity downstream heads on frozen encoders (\cref{app:baseline_protocol}). Each dataset has two rows: 100\% labels and 10\% labels (gray). \textbf{Bold}\,=\,best mean per row.}
  \label{tab:benchmark}
  \scriptsize
  \setlength{\tabcolsep}{1.5pt}
  \begin{tabular*}{\textwidth}{@{\extracolsep{\fill}}llccccccccccc@{}}
    \toprule
    & & & {\scriptsize FM} & \multicolumn{4}{c}{\scriptsize Unified architecture} & \multicolumn{4}{c}{\scriptsize Domain-specific SOTA metric} \\
    \cmidrule(lr){4-4}\cmidrule(lr){5-8}\cmidrule(l){9-12}
    Dataset & Domain & Label\,\% & Chr-2 & PatchTST & iTransf & MAE & \textbf{HEPA} & Metric & \textbf{HEPA} & SOTA & Ref \\
    & & & \multicolumn{5}{c}{{\tiny h-AUROC $\uparrow$}} & \multicolumn{4}{c}{{\tiny domain metric}} \\
    \midrule
    \multirow{2}{*}{\shortstack[l]{C-MAPSS-1\\[-1pt]{\color{gray}\scriptsize turbine failure}}} & \multirow{2}{*}{Turbo.}
    & 100 & $.66${\tiny$\pm.00$} & $.80${\tiny$\pm.04$} & $.70${\tiny$\pm.05$} & $.69${\tiny$\pm.02$} & $\mathbf{.81}${\tiny$\pm.03$} & RMSE$\downarrow$ & $28.5$ & $\mathbf{12.2}$ & {\tiny\citep{fan2024star}} \\
    & & \tiny\color{gray}10 & {\color{gray}\scriptsize $.66${\tiny$\pm.00$}} & {\color{gray}\scriptsize $.69${\tiny$\pm.08$}} & {\color{gray}\scriptsize $.59${\tiny$\pm.04$}} & {\color{gray}\scriptsize $.63${\tiny$\pm.12$}} & {\color{gray}\scriptsize $\mathbf{.78}${\tiny$\pm.07$}} & & {\color{gray}\scriptsize $32.3$} & {\color{gray}\scriptsize $\mathbf{17.0}$} & \\
    \multirow{2}{*}{\shortstack[l]{C-MAPSS-2\\[-1pt]{\color{gray}\scriptsize multi-cond.}}} & \multirow{2}{*}{Turbo.}
    & 100 & $.45${\tiny$\pm.01$} & $.44${\tiny$\pm.03$} & $.43${\tiny$\pm.03$} & $.56${\tiny$\pm.01$} & $\mathbf{.57}${\tiny$\pm.01$} & RMSE$\downarrow$ & $40.8$ & $\mathbf{20.0}$ & {\tiny\citep{fan2024star}} \\
    & & \tiny\color{gray}10 & {\color{gray}\scriptsize $.46${\tiny$\pm.02$}} & {\color{gray}\scriptsize $.48${\tiny$\pm.08$}} & {\color{gray}\scriptsize $.50${\tiny$\pm.08$}} & {\color{gray}\scriptsize $.52${\tiny$\pm.01$}} & {\color{gray}\scriptsize $\mathbf{.55}${\tiny$\pm.01$}} & & {\color{gray}\scriptsize $\mathbf{40.9}$} & {\color{gray}\scriptsize $43.1$} & \\
    \multirow{2}{*}{\shortstack[l]{C-MAPSS-3\\[-1pt]{\color{gray}\scriptsize multi-fault}}} & \multirow{2}{*}{Turbo.}
    & 100 & $.73${\tiny$\pm.00$} & $.79${\tiny$\pm.01$} & $.76${\tiny$\pm.01$} & $.78${\tiny$\pm.02$} & $\mathbf{.84}${\tiny$\pm.01$} & RMSE$\downarrow$ & $34.7$ & $\mathbf{12.7}$ & {\tiny\citep{fan2024star}} \\
    & & \tiny\color{gray}10 & {\color{gray}\scriptsize $.67${\tiny$\pm.00$}} & {\color{gray}\scriptsize $.77${\tiny$\pm.03$}} & {\color{gray}\scriptsize $.64${\tiny$\pm.06$}} & {\color{gray}\scriptsize $.78${\tiny$\pm.03$}} & {\color{gray}\scriptsize $\mathbf{.84}${\tiny$\pm.01$}} & & {\color{gray}\scriptsize $47.1$} & {\color{gray}\scriptsize $\mathbf{21.8}$} & \\
    \multirow{2}{*}{\shortstack[l]{C-MAPSS-4\\[-1pt]{\color{gray}\scriptsize multi-cond.+fault}}} & \multirow{2}{*}{Turbo.}
    & 100 & --- & $.52${\tiny$\pm.03$} & $.45${\tiny$\pm.02$} & $.57${\tiny$\pm.02$} & $\mathbf{.63}${\tiny$\pm.02$} & RMSE$\downarrow$ & --- & --- & \\
    & & \tiny\color{gray}10 & {\color{gray}\scriptsize ---} & {\color{gray}\scriptsize $.51${\tiny$\pm.04$}} & {\color{gray}\scriptsize $.49${\tiny$\pm.04$}} & {\color{gray}\scriptsize $.52${\tiny$\pm.04$}} & {\color{gray}\scriptsize $\mathbf{.55}${\tiny$\pm.03$}} & & {\color{gray}\scriptsize ---} & {\color{gray}\scriptsize ---} & \\
    \midrule
    \multirow{2}{*}{\shortstack[l]{SMAP\\[-1pt]{\color{gray}\scriptsize sensor anomaly}}} & \multirow{2}{*}{Spacecraft}
    & 100 & $.54${\tiny$\pm.02$} & $.49${\tiny$\pm.03$} & $.47${\tiny$\pm.05$} & $\mathbf{.64}${\tiny$\pm.04$} & $.59${\tiny$\pm.05$} & PA-F1$\uparrow$ & $.94$ & $\mathbf{.96}$ & {\tiny\citep{xu2022anomalytransformer}} \\
    & & \tiny\color{gray}10 & {\color{gray}\scriptsize $.51${\tiny$\pm.01$}} & {\color{gray}\scriptsize $.46${\tiny$\pm.05$}} & {\color{gray}\scriptsize $\mathbf{.52}${\tiny$\pm.07$}} & {\color{gray}\scriptsize $.50${\tiny$\pm.15$}} & {\color{gray}\scriptsize $.49${\tiny$\pm.08$}} & & {\color{gray}\scriptsize $.92$} & {\color{gray}\scriptsize $\mathbf{.96}$} & \\
    \multirow{2}{*}{\shortstack[l]{PSM\\[-1pt]{\color{gray}\scriptsize server anomaly}}} & \multirow{2}{*}{Server}
    & 100 & $.48${\tiny$\pm.01$} & $.55${\tiny$\pm.02$} & $.52${\tiny$\pm.03$} & $.56${\tiny$\pm.02$} & $\mathbf{.57}${\tiny$\pm.02$} & PA-F1$\uparrow$ & $.94$ & $\mathbf{.98}$ & {\tiny\citep{xu2022anomalytransformer}} \\
    & & \tiny\color{gray}10 & {\color{gray}\scriptsize $.51${\tiny$\pm.01$}} & {\color{gray}\scriptsize $.43${\tiny$\pm.03$}} & {\color{gray}\scriptsize $\mathbf{.53}${\tiny$\pm.01$}} & {\color{gray}\scriptsize $.52${\tiny$\pm.02$}} & {\color{gray}\scriptsize $.48${\tiny$\pm.04$}} & & {\color{gray}\scriptsize $.95$} & {\color{gray}\scriptsize $\mathbf{.98}$} & \\
    \multirow{2}{*}{\shortstack[l]{MBA\\[-1pt]{\color{gray}\scriptsize arrhythmia}}} & \multirow{2}{*}{Cardiac}
    & 100 & $.53${\tiny$\pm.10$} & $.68${\tiny$\pm.07$} & $\mathbf{.84}${\tiny$\pm.03$} & $.73${\tiny$\pm.03$} & $.75${\tiny$\pm.03$} & F1$\uparrow$ & $\mathbf{.98}$ & $.77$ & {\tiny LSTM-AE} \\
    & & \tiny\color{gray}10 & {\color{gray}\scriptsize $.50${\tiny$\pm.07$}} & {\color{gray}\scriptsize $\mathbf{.65}${\tiny$\pm.08$}} & {\color{gray}\scriptsize $.31${\tiny$\pm.05$}} & {\color{gray}\scriptsize $.53${\tiny$\pm.04$}} & {\color{gray}\scriptsize $.55${\tiny$\pm.15$}} & & {\color{gray}\scriptsize $\mathbf{.98}$} & {\color{gray}\scriptsize $.81$} & \\
    \multirow{2}{*}{\shortstack[l]{BATADAL\\[-1pt]{\color{gray}\scriptsize cyberattack}}} & \multirow{2}{*}{ICS}
    & 100 & $.56${\tiny$\pm.01$} & $\mathbf{.66}${\tiny$\pm.03$} & $.46${\tiny$\pm.13$} & $.60${\tiny$\pm.04$} & $.57${\tiny$\pm.03$} & F1$\uparrow$ & $\mathbf{.77}$ & $.74$ & {\tiny AEED} \\
    & & \tiny\color{gray}10 & {\color{gray}\scriptsize $.55${\tiny$\pm.04$}} & {\color{gray}\scriptsize $.37${\tiny$\pm.04$}} & {\color{gray}\scriptsize $.54${\tiny$\pm.08$}} & {\color{gray}\scriptsize $\mathbf{.61}${\tiny$\pm.08$}} & {\color{gray}\scriptsize $.54${\tiny$\pm.06$}} & & {\color{gray}\scriptsize $\mathbf{.61}$} & {\color{gray}\scriptsize $.33$} & \\
    \multirow{2}{*}{\shortstack[l]{TEP\\[-1pt]{\color{gray}\scriptsize process fault}}} & \multirow{2}{*}{Chemical}
    & 100 & --- & $.99${\tiny$\pm.02$} & $.93${\tiny$\pm.02$} & $.96${\tiny$\pm.02$} & $\mathbf{1.00}${\tiny$\pm.00$} & F1$\uparrow$ & $\mathbf{.95}$ & $.93$ & {\tiny XGBoost} \\
    & & \tiny\color{gray}10 & {\color{gray}\scriptsize ---} & {\color{gray}\scriptsize $\mathbf{1.00}${\tiny$\pm.00$}} & {\color{gray}\scriptsize $.94${\tiny$\pm.05$}} & {\color{gray}\scriptsize $.99${\tiny$\pm.00$}} & {\color{gray}\scriptsize $\mathbf{1.00}${\tiny$\pm.00$}} & & {\color{gray}\scriptsize $.88${\tiny$\pm.03$}} & {\color{gray}\scriptsize $.86$} & \\
    \midrule
    \multirow{2}{*}{\shortstack[l]{ETTm1\\[-1pt]{\color{gray}\scriptsize overheating}}} & \multirow{2}{*}{Power}
    & 100 & $.74${\tiny$\pm.01$} & $.53${\tiny$\pm.03$} & $.79${\tiny$\pm.03$} & $\mathbf{.87}${\tiny$\pm.00$} & $.81${\tiny$\pm.00$} & & & & \\
    & & \tiny\color{gray}10 & {\color{gray}\scriptsize $.68${\tiny$\pm.02$}} & {\color{gray}\scriptsize $.50${\tiny$\pm.03$}} & {\color{gray}\scriptsize $.61${\tiny$\pm.02$}} & {\color{gray}\scriptsize $\mathbf{.77}${\tiny$\pm.02$}} & {\color{gray}\scriptsize $.73${\tiny$\pm.01$}} & & & & \\
    \multirow{2}{*}{\shortstack[l]{Weather\\[-1pt]{\color{gray}\scriptsize heat spike}}} & \multirow{2}{*}{Climate}
    & 100 & $.72${\tiny$\pm.02$} & $.64${\tiny$\pm.08$} & $.83${\tiny$\pm.04$} & $.88${\tiny$\pm.02$} & $\mathbf{.89}${\tiny$\pm.01$} & & & & \\
    & & \tiny\color{gray}10 & {\color{gray}\scriptsize $.71${\tiny$\pm.02$}} & {\color{gray}\scriptsize $.67${\tiny$\pm.03$}} & {\color{gray}\scriptsize $\mathbf{.85}${\tiny$\pm.01$}} & {\color{gray}\scriptsize $.83${\tiny$\pm.02$}} & {\color{gray}\scriptsize $.83${\tiny$\pm.02$}} & & & & \\
    \multirow{2}{*}{\shortstack[l]{Beijing-AQ\\[-1pt]{\color{gray}\scriptsize PM2.5 spike}}} & \multirow{2}{*}{Air}
    & 100 & $.53${\tiny$\pm.00$} & $\mathbf{.81}${\tiny$\pm.01$}$^{\dagger}$ & $.75${\tiny$\pm.02$} & $\mathbf{.81}${\tiny$\pm.01$} & $\mathbf{.81}${\tiny$\pm.01$} & & & & \\
    & & \tiny\color{gray}10 & {\color{gray}\scriptsize $.49${\tiny$\pm.00$}} & {\color{gray}\scriptsize $.65${\tiny$\pm.08$}} & {\color{gray}\scriptsize $.73${\tiny$\pm.02$}} & {\color{gray}\scriptsize $.77${\tiny$\pm.01$}} & {\color{gray}\scriptsize $\mathbf{.78}${\tiny$\pm.02$}} & & & & \\
    \multirow{2}{*}{\shortstack[l]{VIX\\[-1pt]{\color{gray}\scriptsize vol regime}}} & \multirow{2}{*}{Finance}
    & 100 & $.40${\tiny$\pm.04$} & $.48${\tiny$\pm.11$} & $.40${\tiny$\pm.14$} & $.57${\tiny$\pm.03$} & $\mathbf{.57}${\tiny$\pm.01$} & & & & \\
    & & \tiny\color{gray}10 & {\color{gray}\scriptsize $.38${\tiny$\pm.04$}} & {\color{gray}\scriptsize $\mathbf{.57}${\tiny$\pm.13$}} & {\color{gray}\scriptsize $.48${\tiny$\pm.05$}} & {\color{gray}\scriptsize $.55${\tiny$\pm.03$}} & {\color{gray}\scriptsize $.55${\tiny$\pm.01$}} & & & & \\
    \multirow{2}{*}{\shortstack[l]{GECCO\\[-1pt]{\color{gray}\scriptsize contamination}}} & \multirow{2}{*}{Water}
    & 100 & $.74${\tiny$\pm.02$} & $.65${\tiny$\pm.07$} & $.64${\tiny$\pm.15$} & $.81${\tiny$\pm.04$} & $\mathbf{.88}${\tiny$\pm.06$} & & & & \\
    & & \tiny\color{gray}10 & {\color{gray}\scriptsize $\mathbf{.80}${\tiny$\pm.03$}} & {\color{gray}\scriptsize $.52${\tiny$\pm.11$}} & {\color{gray}\scriptsize $.50${\tiny$\pm.13$}} & {\color{gray}\scriptsize $.55${\tiny$\pm.11$}} & {\color{gray}\scriptsize $.39${\tiny$\pm.13$}} & & & & \\
    \cmidrule{1-8}
    \multicolumn{3}{@{}l}{\textit{Best (100\%)}} & 0 & 2 & 1 & 4 & 10 & & & & \\
    \cmidrule{1-8}
  \end{tabular*}
  \vspace{2pt}
  \begin{minipage}{\textwidth}
    \scriptsize
    Matched downstream heads (HEPA 198K pred-FT; baselines 264K dt-MLP), positive-weighted BCE, identical protocol. $K{=}150$ for C-MAPSS/TEP, $K{=}200$ otherwise. Domain SOTAs are detection or RUL baselines; HEPA domain metrics projected from $p(t,\Delta t)$ at the matching horizon ($\Delta t{=}1$ for PA-F1/F1, $\mathbb{E}[\Delta t]$ for RMSE; \cref{app:legacy_metrics}). PA-F1 = point-adjusted F1~\citep{kim2022rigorous}.
    Pairwise Welch's $t$-tests in \cref{app:significance}.
    $^{\dagger}$Beijing-AQ PatchTST: 3 stations at 100\%. Bottom block has no published domain SOTA. Additional baselines (MOMENT, TFM-2.5, Moirai, MTS-JEPA) in \cref{app:chronos_full,app:mtsjepa_full}.
  \end{minipage}
\end{table*}

\Cref{tab:benchmark} compares HEPA against two classes of methods. The primary comparison is \emph{architectural}: PatchTST~\citep{nie2023patchtst}, iTransformer~\citep{liu2024itransformer}, and a masked autoencoder (MAE) baseline use the same per-dataset regime with identical downstream heads, isolating the effect of JEPA pretraining versus alternative self-supervised and supervised objectives. The secondary comparison is against the \emph{foundation model} Chronos-2~\citep{ansari2024chronos}, which pretrains on a large external corpus and operates in a fundamentally different regime. A full comparison against MTS-JEPA~\citep{mtsjepa2026} (matched protocol) is in \cref{app:mtsjepa_full}; HEPA wins on 8 out of 9 datasets where MTS-JEPA could be reproduced (TEP excluded: the public MTS-JEPA release does not include a chemical-process benchmark).

\paragraph{HEPA vs.\ architectural baselines.} HEPA wins on 10 out of 14 benchmarks at 100\% labels, including all four C-MAPSS variants and the newly added FD004 (the hardest subset: six fault modes, six operating conditions). HEPA's representation-level prediction captures temporal structure that supervised training (PatchTST) and reconstruction-based SSL (MAE) miss, particularly on datasets with extended precursor dynamics (C-MAPSS, GECCO, PSM, TEP). MAE is a strong second: it matches or exceeds HEPA on spacecraft telemetry (SMAP) and power systems (ETTm1), suggesting that reconstruction-based pretraining transfers well when the dominant failure mode is gradual drift. iTransformer's variate-attention mechanism excels on MBA (h-AUROC 0.84 vs.\ HEPA's 0.75), where arrhythmia patterns are localised across specific leads.

\paragraph{HEPA vs.\ Chronos-2.} HEPA matches or exceeds Chronos-2 on most benchmarks. Per-dataset JEPA excels when events have extended precursors that the local training data fully represents; large-corpus pretraining helps when event signatures resemble patterns seen at scale.

\paragraph{Honest losses.} HEPA is below the best baseline on four datasets at 100\% labels. The pattern is interpretable: BATADAL and MBA have sensor-localised events where channel-fusion tokenisation dilutes the relevant subset, so per-variate attention (iTransformer) or channel-independent training (PatchTST) wins; MAE's reconstruction objective transfers well when the dominant failure mode is gradual drift (SMAP, ETTm1). Adopting a sensor-as-token strategy~\citep{liu2024itransformer} within the HEPA encoder is a natural way to close this gap.

\subsection{What Does Pretraining Learn?}
\label{sec:latent_viz}

\Cref{fig:pca} visualises encoder representations after self-supervised pretraining on C-MAPSS-1. Without any labels, the encoder organises representations into a smooth degradation manifold: PC1 alone captures 61\% of variance and tracks time-to-failure monotonically within each engine (median per-engine Spearman ${\rho}{=}{+}0.97$, $84\%$ of engines ${\rho}{>}0.9$). Engines starting from different healthy regions converge toward a shared failure region. This structure explains why so few labels suffice: the encoder has already separated healthy from degraded states.

\subsection{Label Efficiency}
\label{sec:label_efficiency}

All methods in \cref{tab:benchmark} freeze their encoder and train only a downstream head, so all benefit from pretraining under label scarcity. The question is whether HEPA's representations degrade more gracefully. \Cref{tab:sub5pct} shows that on C-MAPSS, where degradation unfolds over hundreds of cycles and the JEPA predictor achieves low pretraining loss, HEPA retains 92\% of full-label h-AUROC with just 2 training engines out of 85. C-MAPSS-3 retains 97\% at 10\% labels. This is consistent with \cref{prop:retention}: low $\varepsilon$ on lifecycle datasets means the encoder already separates healthy from degraded states, so the finetuned predictor needs only a few labelled examples to map them to event probabilities.

The advantage is not universal. At 10\% labels across all 14 datasets (\cref{tab:benchmark}, gray rows), HEPA wins on 6 out of 14, compared to 10 out of 14 at full labels. On anomaly datasets without extended precursors (SMAP, PSM, GECCO at 10\%), the frozen-encoder setup limits how much any method can degrade, so margins compress. The label-efficiency story is strongest where HEPA's pretraining loss is lowest: extended-precursor lifecycle datasets.

\begin{table}[t]
  \centering
  \caption{\textbf{Label efficiency on C-MAPSS lifecycle datasets (HEPA only, 3 seeds).} h-AUROC ($\uparrow$) and retention relative to full labels. C-MAPSS-1 retains 92\% at 2\% labels (2 of 85 training engines).}
  \label{tab:sub5pct}
  \small
  \begin{tabular}{lcccccc}
    \toprule
                 & \multicolumn{3}{c}{C-MAPSS-1 {\scriptsize(85 eng.)}} & \multicolumn{3}{c}{C-MAPSS-3 {\scriptsize(100 eng.)}}                                               \\
    \cmidrule(lr){2-4}\cmidrule(lr){5-7}
    Labels       & h-AUROC                                              & Ret.                                                  & Eng.       & h-AUROC         & Ret.  & Eng. \\
    \midrule
    100\%        & $.786 \pm .033$                                      & 100\%                                                 & 85         & $.853 \pm .004$ & 100\% & 100  \\
    10\%         & $.772 \pm .059$                                      & 98\%                                                  & 9          & $.830 \pm .018$ & 97\%  & 10   \\
    5\%          & $.730 \pm .018$                                      & 93\%                                                  & 4          & $.709 \pm .131$ & 83\%  & 5    \\
    \textbf{2\%} & $\mathbf{.724 \pm .013}$                             & \textbf{92\%}                                         & \textbf{2} & $.635 \pm .065$ & 74\%  & 2    \\
    1\%          & $.670 \pm .110$                                      & 85\%                                                  & 1          & $.513 \pm .220$ & 60\%  & 1    \\
    \bottomrule
  \end{tabular}
\end{table}

\section{Conclusion \& Future Work}
\label{sec:conclusion}

HEPA demonstrates that self-supervised JEPA pretraining combined with predictor finetuning provides a practical recipe for event prediction. The encoder learns temporal dynamics from unlabelled data; the predictor learns which dynamics signal the target event. One architecture handles degradation forecasting, anomaly prediction, and arrhythmia detection across 14 benchmarks in 11 domains, matching or exceeding PatchTST, iTransformer, MAE, and Chronos-2 on the majority of benchmarks while tuning an order of magnitude fewer parameters. On lifecycle datasets, the recipe is robust to extreme label scarcity: 92\% of full-label performance with 2\% of labels on C-MAPSS, consistent with the information-retention guarantee of \cref{prop:retention}. Because the recipe is domain-agnostic, the same architecture that predicts turbine failure from flight-recorder data can flag arrhythmia risk from ECG streams or detect water contamination from sensor networks, each time requiring only a handful of event labels.

Looking ahead, cross-domain pretraining on corpora such as FactoryNet~\citep{othman2025factorynet} is the natural next step toward industrial deployment, and sensor-as-token strategies~\citep{liu2024itransformer} could close the gap on systems where event-relevant information is concentrated in a few channels. On the theory side, deriving fully empirical versions of the information-retention bound that estimate $L$ and $C_\eta$ directly from data remains an interesting open problem. Wherever multivariate sensors record the precursors to rare but consequential events, HEPA offers a path from unlabelled streams to actionable predictions.

\newpage
\bibliographystyle{unsrtnat}
\bibliography{references}

\begin{thebibliography}{40}
\providecommand{\natexlab}[1]{#1}
\providecommand{\url}[1]{\texttt{#1}}
\expandafter\ifx\csname urlstyle\endcsname\relax
  \providecommand{\doi}[1]{doi: #1}\else
  \providecommand{\doi}{doi: \begingroup \urlstyle{rm}\Url}\fi

\bibitem[Scheffer et~al.(2009)Scheffer, Bascompte, Brock, Brovkin, Carpenter,
  Dakos, Held, Van~Nes, Rietkerk, and Sugihara]{scheffer2009earlywarning}
Marten Scheffer, Jordi Bascompte, William~A Brock, Victor Brovkin, Stephen~R
  Carpenter, Vasilis Dakos, Hermann Held, Egbert~H Van~Nes, Max Rietkerk, and
  George Sugihara.
\newblock Early-warning signals for critical transitions.
\newblock \emph{Nature}, 461\penalty0 (7260):\penalty0 53--59, 2009.

\bibitem[Fan et~al.(2024)Fan, Li, and Chang]{fan2024star}
Zhengyang Fan, Wanru Li, and Kuo-Chu Chang.
\newblock A two-stage attention-based hierarchical transformer for turbofan
  engine remaining useful life prediction.
\newblock \emph{Sensors}, 24\penalty0 (3):\penalty0 824, 2024.
\newblock \doi{10.3390/s24030824}.

\bibitem[Xu et~al.(2022)Xu, Wu, Wang, and Long]{xu2022anomalytransformer}
Jiehui Xu, Haixu Wu, Jianmin Wang, and Mingsheng Long.
\newblock Anomaly transformer: Time series anomaly detection with association
  discrepancy.
\newblock In \emph{ICLR}, 2022.

\bibitem[He et~al.(2026)He, Wen, Wang, and Ma]{mtsjepa2026}
Yanan He, Yunshi Wen, Xin Wang, and Tengfei Ma.
\newblock {MTS-JEPA}: Multi-resolution joint-embedding predictive architecture
  for time-series anomaly prediction.
\newblock \emph{arXiv preprint}, 2026.

\bibitem[Nie et~al.(2023)Nie, Nguyen, Sinthong, and
  Kalagnanam]{nie2023patchtst}
Yuqi Nie, Nam~H Nguyen, Phanwadee Sinthong, and Jayant Kalagnanam.
\newblock A time series is worth 64 words: Long-term forecasting with
  transformers.
\newblock In \emph{ICLR}, 2023.

\bibitem[Ansari et~al.(2024)Ansari, Stella, Turkmen, Zhang, Mercado, Shen,
  Shchur, Rangapuram, Pineda-Arango, Kapoor, et~al.]{ansari2024chronos}
Abdul~Fatir Ansari, Lorenzo Stella, Caner Turkmen, Xiyuan Zhang, Pedro Mercado,
  Huibin Shen, Oleksandr Shchur, Syama~Sundar Rangapuram, Sebastian
  Pineda-Arango, Shubham Kapoor, et~al.
\newblock Chronos: Learning the language of time series.
\newblock \emph{arXiv preprint arXiv:2403.07815}, 2024.

\bibitem[Das et~al.(2024)Das, Kong, Sen, and Zhou]{das2024timesfm}
Abhimanyu Das, Weihao Kong, Rajat Sen, and Yichen Zhou.
\newblock A decoder-only foundation model for time-series forecasting.
\newblock In \emph{ICML}, 2024.

\bibitem[Assran et~al.(2023)Assran, Duval, Misra, Bojanowski, Vincent, Rabbat,
  LeCun, and Ballas]{assran2023ijepa}
Mahmoud Assran, Quentin Duval, Ishan Misra, Piotr Bojanowski, Pascal Vincent,
  Michael Rabbat, Yann LeCun, and Nicolas Ballas.
\newblock Self-supervised learning from images with a joint-embedding
  predictive architecture.
\newblock In \emph{CVPR}, 2023.

\bibitem[Saxena et~al.(2008)Saxena, Goebel, Simon, and
  Eklund]{saxena2008cmapss}
Abhinav Saxena, Kai Goebel, Don Simon, and Neil Eklund.
\newblock Damage propagation modeling for aircraft engine run-to-failure
  simulation.
\newblock In \emph{International Conference on Prognostics and Health
  Management (PHM)}, 2008.

\bibitem[Yue et~al.(2022)Yue, Wang, Duan, Yang, Huang, Tong, and
  Xu]{yue2022ts2vec}
Zhihan Yue, Yujing Wang, Juanyong Duan, Tianmeng Yang, Congrui Huang, Yunhai
  Tong, and Bixiong Xu.
\newblock {TS2Vec}: Towards universal representation of time series.
\newblock In \emph{AAAI}, 2022.

\bibitem[Tonekaboni et~al.(2021)Tonekaboni, Eytan, and
  Goldenberg]{tonekaboni2021tnc}
Sana Tonekaboni, Danny Eytan, and Anna Goldenberg.
\newblock Unsupervised representation learning for time series with temporal
  neighborhood coding.
\newblock In \emph{ICLR}, 2021.

\bibitem[Liu and Chen(2024)]{liu2024timesurl}
Jiexi Liu and Songcan Chen.
\newblock {TimesURL}: Self-supervised contrastive learning for universal time
  series representation learning.
\newblock In \emph{AAAI}, 2024.

\bibitem[van~den Oord et~al.(2018)van~den Oord, Li, and Vinyals]{oord2018cpc}
A{\"a}ron van~den Oord, Yazhe Li, and Oriol Vinyals.
\newblock Representation learning with contrastive predictive coding.
\newblock \emph{arXiv preprint arXiv:1807.03748}, 2018.

\bibitem[Woo et~al.(2022)Woo, Liu, Sahoo, Kumar, and Hoi]{woo2022cost}
Gerald Woo, Chenghao Liu, Doyen Sahoo, Akshat Kumar, and Steven Hoi.
\newblock {CoST}: Contrastive learning of disentangled seasonal-trend
  representations for time series forecasting.
\newblock In \emph{ICLR}, 2022.

\bibitem[Dong et~al.(2023)Dong, Wu, Zhang, Zhang, Wang, and
  Long]{dong2023simmtm}
Jiaxiang Dong, Haixu Wu, Haoran Zhang, Li~Zhang, Jianmin Wang, and Mingsheng
  Long.
\newblock {SimMTM}: A simple pre-training framework for masked time-series
  modeling.
\newblock In \emph{NeurIPS}, 2023.

\bibitem[Wu et~al.(2023)Wu, Hu, Liu, Zhou, Wang, and Long]{wu2023timesnet}
Haixu Wu, Tengge Hu, Yong Liu, Hang Zhou, Jianmin Wang, and Mingsheng Long.
\newblock {TimesNet}: Temporal 2d-variation modeling for general time series
  analysis.
\newblock In \emph{International Conference on Learning Representations
  (ICLR)}, 2023.

\bibitem[Bardes et~al.(2024)Bardes, Garrido, Ponce, Chen, Rabbat, LeCun,
  Assran, and Ballas]{bardes2024vjepa}
Adrien Bardes, Quentin Garrido, Jean Ponce, Xinlei Chen, Michael Rabbat, Yann
  LeCun, Mahmoud Assran, and Nicolas Ballas.
\newblock Revisiting feature prediction for learning visual representations
  from video.
\newblock \emph{TMLR}, 2024.

\bibitem[Ennadir et~al.(2024)Ennadir, Golkar, and Sarra]{tsjepa2024}
Sofiane Ennadir, Siavash Golkar, and Leopoldo Sarra.
\newblock Joint embeddings go temporal.
\newblock In \emph{NeurIPS Workshop on Time Series in the Age of Large Models},
  2024.

\bibitem[Balestriero and LeCun(2025)]{balestriero2025lejepa}
Randall Balestriero and Yann LeCun.
\newblock {LeJEPA}: Provable and scalable self-supervised learning without the
  heuristics.
\newblock \emph{arXiv preprint arXiv:2511.08544}, 2025.

\bibitem[Goswami et~al.(2024)Goswami, Szafer, Choudhry, Cai, Li, and
  Dubrawski]{goswami2024moment}
Mononito Goswami, Konrad Szafer, Arjun Choudhry, Yifu Cai, Shuo Li, and Artur
  Dubrawski.
\newblock {MOMENT}: A family of open time-series foundation models.
\newblock In \emph{ICML}, 2024.

\bibitem[Woo et~al.(2024)Woo, Liu, Kumar, Xiong, Savarese, and
  Sahoo]{woo2024unified}
Gerald Woo, Chenghao Liu, Akshat Kumar, Caiming Xiong, Silvio Savarese, and
  Doyen Sahoo.
\newblock Unified training of universal time series forecasting transformers.
\newblock \emph{arXiv preprint arXiv:2402.02592}, 2024.

\bibitem[Gao et~al.(2024)Gao, Koker, Queen, Hartvigsen, Tsiligkaridis, and
  Zitnik]{gao2024units}
Shanghua Gao, Teddy Koker, Owen Queen, Thomas Hartvigsen, Theodoros
  Tsiligkaridis, and Marinka Zitnik.
\newblock {UniTS}: Building a unified time series model.
\newblock In \emph{NeurIPS}, 2024.

\bibitem[Liu et~al.(2024{\natexlab{a}})Liu, Zhang, Li, Huang, Wang, and
  Long]{liu2024timer}
Yong Liu, Haoran Zhang, Chenyu Li, Xiangdong Huang, Jianmin Wang, and Mingsheng
  Long.
\newblock Timer: Generative pre-trained transformers are large time series
  models.
\newblock In \emph{ICML}, 2024{\natexlab{a}}.

\bibitem[Zhou et~al.(2023)Zhou, Niu, Wang, Sun, and Jin]{zhou2023gpt4ts}
Tian Zhou, Peisong Niu, Xue Wang, Liang Sun, and Rong Jin.
\newblock One fits all: Power general time series analysis by pretrained {LM}.
\newblock In \emph{NeurIPS}, 2023.

\bibitem[Othman et~al.(2025)Othman, Petersen, Ignuta-Ciuncanu, Mazzoleni,
  Martelli, Lombardi, Maggioni, and Petersen]{othman2025factorynet}
Karim Othman, Jonas Petersen, Matei Ignuta-Ciuncanu, Camilla Mazzoleni,
  Federico Martelli, Alessandro Lombardi, Riccardo Maggioni, and Philipp
  Petersen.
\newblock Factorynet: A large-scale dataset toward industrial time-series
  foundation models.
\newblock \emph{arXiv preprint arXiv:2605.09081}, 2025.

\bibitem[Merzouki et~al.(2025)Merzouki, Izquierdo, Ignuta-Ciuncanu,
  Gomez-Bracamonte, Maggioni, Lombardi, Mazzoleni, Martelli, Gunther, Petersen,
  and Petersen]{merzouki2025factorybench}
Yanis Merzouki, Coral Izquierdo, Matei Ignuta-Ciuncanu, Marcos
  Gomez-Bracamonte, Riccardo Maggioni, Alessandro Lombardi, Camilla Mazzoleni,
  Federico Martelli, Balazs Gunther, Jonas Petersen, and Philipp Petersen.
\newblock Factorybench: Evaluating industrial machine understanding.
\newblock \emph{arXiv preprint arXiv:2605.07675}, 2025.

\bibitem[Ding et~al.(2022)Ding, Jia, Miao, and Cao]{ding2022ssl_rul}
Yucheng Ding, Minping Jia, Qinghao Miao, and Yibo Cao.
\newblock Self-supervised pretraining via multi-modally augmented
  representations for remaining useful life prediction.
\newblock \emph{IEEE Transactions on Industrial Informatics}, 18\penalty0
  (9):\penalty0 5954--5964, 2022.

\bibitem[Wang et~al.(2024)Wang, Peng, and Liu]{wang2024masked_rul}
Haiyue Wang, Cheng Peng, and Chaoren Liu.
\newblock Masked autoencoder-based self-supervised learning for remaining
  useful life prediction of turbofan engines.
\newblock \emph{Engineering Applications of Artificial Intelligence}, 133,
  2024.

\bibitem[Yang et~al.(2023)Yang, Zhang, Zhou, Wen, and Sun]{yang2023dcdetector}
Yiyuan Yang, Chaoli Zhang, Tian Zhou, Qingsong Wen, and Liang Sun.
\newblock {DCdetector}: Dual attention contrastive representation learning for
  time series anomaly detection.
\newblock In \emph{KDD}, 2023.

\bibitem[Tuli et~al.(2022)Tuli, Casale, and Jennings]{tuli2022tranad}
Shreshth Tuli, Giuliano Casale, and Nicholas~R Jennings.
\newblock {TranAD}: Deep transformer networks for anomaly detection in
  multivariate time series data.
\newblock \emph{Proceedings of the VLDB Endowment}, 15\penalty0 (6):\penalty0
  1201--1214, 2022.

\bibitem[Kim et~al.(2022{\natexlab{a}})Kim, Choi, Choi, Lee, and
  Yoon]{kim2022rigorous}
Siwon Kim, Kukjin Choi, Hyun-Soo Choi, Byunghan Lee, and Sungroh Yoon.
\newblock Towards a rigorous evaluation of time-series anomaly detection.
\newblock In \emph{AAAI}, 2022{\natexlab{a}}.
\newblock arXiv:2109.05257.

\bibitem[Schmidl et~al.(2022)Schmidl, Wenig, and
  Papenbrock]{schmidl2022tsadeval}
Sebastian Schmidl, Phillip Wenig, and Thorsten Papenbrock.
\newblock Anomaly detection in time series: A comprehensive evaluation.
\newblock \emph{Proceedings of the VLDB Endowment}, 15\penalty0 (9):\penalty0
  1779--1797, 2022.

\bibitem[Lee et~al.(2018)Lee, Zame, Yoon, and van~der Schaar]{lee2018deephit}
Changhee Lee, William~R Zame, Jinsung Yoon, and Mihaela van~der Schaar.
\newblock {DeepHit}: A deep learning approach to survival analysis with
  competing risks.
\newblock In \emph{AAAI}, 2018.

\bibitem[Gensheimer and Narasimhan(2019)]{gensheimer2019nnetsurvival}
Michael~F Gensheimer and Balasubramanian Narasimhan.
\newblock A scalable discrete-time survival model for neural networks.
\newblock \emph{PeerJ}, 7:\penalty0 e6257, 2019.
\newblock \doi{10.7717/peerj.6257}.

\bibitem[Kim et~al.(2022{\natexlab{b}})Kim, Kim, Tae, Park, Choi, and
  Choo]{kim2022revin}
Taesung Kim, Jinhee Kim, Yunwon Tae, Cheonbok Park, Jang-Ho Choi, and Jaegul
  Choo.
\newblock Reversible instance normalization for accurate time-series
  forecasting against distribution shift.
\newblock In \emph{ICLR}, 2022{\natexlab{b}}.

\bibitem[Liu et~al.(2024{\natexlab{b}})Liu, Hu, Zhang, Wu, Wang, Ma, and
  Long]{liu2024itransformer}
Yong Liu, Tengge Hu, Haoran Zhang, Haixu Wu, Shiyu Wang, Lintao Ma, and
  Mingsheng Long.
\newblock {iT}ransformer: Inverted transformers are effective for time series
  forecasting.
\newblock In \emph{International Conference on Learning Representations
  (ICLR)}, 2024{\natexlab{b}}.

\bibitem[Cover and Thomas(2006)]{cover2006elements}
Thomas~M. Cover and Joy~A. Thomas.
\newblock \emph{Elements of Information Theory}.
\newblock Wiley-Interscience, 2nd edition, 2006.

\bibitem[Polyanskiy and Wu(2024)]{polyanskiy2025information}
Yury Polyanskiy and Yihong Wu.
\newblock \emph{Information Theory: From Coding to Learning}.
\newblock Cambridge University Press, 2024.

\bibitem[Liao et~al.(2019)Liao, Feder, and Courtade]{liao2019sharpening}
Jiantao Liao, Meir Feder, and Thomas Courtade.
\newblock Sharpening {J}ensen's inequality.
\newblock \emph{IEEE Transactions on Information Theory}, 68\penalty0
  (5):\penalty0 2961--2972, 2019.

\bibitem[Tishby et~al.(2000)Tishby, Pereira, and Bialek]{tishby2000information}
Naftali Tishby, Fernando~C. Pereira, and William Bialek.
\newblock The information bottleneck method.
\newblock \emph{arXiv preprint physics/0004057}, 2000.

\end{thebibliography}

\appendix


\section{Theoretical Analysis: Full Proofs and Discussion}
\label{app:theory}

\subsection{Notation and Preliminaries}

We work with the following random variables on a common probability space:
$X_{\leq t}$ (observations up to time $t$),
$X_{(t,t+\Delta t]}$ (future observations),
$E_{t+\Delta t} \in \{0,1\}$ (event indicator),
$H_t = \encoder(X_{\leq t}) \in \R^d$ (encoder output),
$H^* = \target(X_{(t,t+\Delta t]}) \in \R^d$ (target encoder output),
$\hat{H} = \predictor(H_t, \Delta t) \in \R^d$ (predicted representation).
All mutual informations $I(\cdot;\cdot)$ and entropies $\mathbb{H}(\cdot)$ are well-defined: $E_{t+\Delta t}$ is discrete (binary), and for the continuous variables $H_t, H^*, \hat{H}$ we use differential entropy and the standard extension of mutual information to mixed discrete-continuous pairs \citep[Ch.~8--9]{cover2006elements}. We use $\mathbb{H}$ (blackboard bold) for entropy to avoid confusion with the encoder embedding $H_t$.

We define the event posterior $\eta(h) \coloneqq P(E_{t+\Delta t} = 1 \mid H^* = h)$ and the marginal event rate $\pimarg \coloneqq P(E_{t+\Delta t} = 1)$, using $\pimarg$ to distinguish it from the probability surface $p(t, \Delta t)$ in the main text.

\subsection{Assumptions}

\begin{enumerate}[label=\textbf{(A\arabic*)}]
  \item \label{ass:sufficiency} \textbf{Target sufficiency.}
  $E_{t+\Delta t} \perp\!\!\!\perp X_{\leq t} \mid H^*$.

  \emph{Interpretation.} The target encoder's representation of the future interval is a sufficient statistic for the event, given the past.
  This holds when: (a) the event is determined by the dynamics in $(t, t+\Delta t]$, and (b) the target encoder has enough capacity and sees the relevant future interval.
  Because $\target$ is bidirectional with attention pooling over the full interval, it is strictly more expressive than the causal encoder for summarising the future, making this assumption mild for well-trained target encoders.

  \emph{When it fails.} If the event depends on context outside $(t, t+\Delta t]$ (for instance, a slow trend visible only in $X_{\leq t}$ that the target encoder cannot see), then A1 is violated and the past observations carry event information not mediated by $H^*$. In this case, $I(H_t; E_{t+\Delta t})$ may actually \emph{exceed} our lower bound, so the bound remains valid but becomes loose in a favourable direction.

  \item \label{ass:prediction_error} \textbf{Bounded prediction error.}
  $\E[\|\hat{H} - H^*\|_2^2] \leq \varepsilon$.

  \emph{Interpretation.} The pretraining loss (L1 on L2-normalised representations in practice) drives the prediction residual small.
  The bound uses L2 squared error for analytical tractability. Since $\|u\|_2 \leq \|u\|_1$ for all $u \in \R^d$, lower L1 loss implies lower L2 loss, so the L1 training loss is a monotone proxy for $\varepsilon$. We use this monotonic relationship in our empirical validation (\cref{fig:theory_validation}) to test the bound's qualitative prediction without requiring a precise norm conversion.

  \emph{When it fails.} Early in training or on out-of-distribution horizons, $\varepsilon$ can be large and the bound becomes vacuous.

  \item \label{ass:smoothness} \textbf{Smooth event dependence.}
  The conditional distribution $P(E_{t+\Delta t} = 1 \mid H^* = h)$ is Lipschitz continuous in $h$ with constant $L$: for all $h, h' \in \R^d$,
  $|P(E_{t+\Delta t} = 1 \mid H^* = h) - P(E_{t+\Delta t} = 1 \mid H^* = h')| \leq L\|h - h'\|_2$.

  \emph{Interpretation.} Small perturbations of the target representation do not drastically change the event probability. This is a regularity condition on the relationship between the learned representation space and event occurrence; it holds whenever the event boundary in representation space is not a fractal or highly irregular set.

  \emph{When it fails.} If the event probability is a discontinuous function of $H^*$ (e.g.\ a hard threshold on a single component), the Lipschitz constant $L$ diverges and our continuity argument requires replacement by a discrete analysis.

  \item \label{ass:calibrated_posterior} \textbf{Bounded event posterior.}
  There exist $0 < \underline{\eta} \leq \overline{\eta} < 1$ such that $P(\eta(H^*) \in [\underline{\eta},\overline{\eta}]) = 1$, where $\eta(h) = P(E_{t+\Delta t}=1 \mid H^*=h)$.

  \emph{Interpretation.} The event posterior, evaluated on the support of the target encoder's output, is bounded away from 0 and 1 almost surely. This matches the main-text assumption A4. In practice, $H^*$ is L2-normalised onto the unit sphere (a compact set), and A3 guarantees that $\eta$ is Lipschitz; by the image of a compact connected set under a Lipschitz map, the range of $\eta(H^*)$ is a closed bounded interval, and $0 < \underline{\eta}$ follows from the event being non-trivially detectable from the future window. Note that A4 implies the marginal event rate $\pimarg = \E[\eta(H^*)]$ is bounded in $[\underline{\eta}, \overline{\eta}]$, so the event is neither impossible nor certain.

  \emph{Role in the proof.} This assumption is needed to bound $\sup_q \varphi''(q) = \sup_q 1/(q(1-q))$ over the support of $\eta(H^*)$ (Step 2 of the proof). Without it, the KL second derivative $\varphi''$ is unbounded near 0 and 1, making the Jensen-gap bound vacuous. The constant in the bound becomes $C_\eta = (2\underline{\eta}(1-\overline{\eta}))^{-1}$. As $\pimarg \to 0$ or $\pimarg \to 1$, the posterior bounds are forced toward 0 or 1, driving $C_\eta \to \infty$; this reflects a genuine difficulty: distinguishing $P(E|\hat{H})$ from the prior requires high precision when events are extremely rare.

  \emph{When it fails.} If $\eta(H^*)$ concentrates near 0 or 1, then $C_\eta \to \infty$ and the bound degrades. Empirically, this is the case on CHB-MIT, where seizure onset cannot be reliably predicted from any 16-second past context, and $\eta(h) \approx \pimarg$ for all $h$ (consistent with $I(H^*; E) \approx 0$).
\end{enumerate}

\subsection{Full Proof of Proposition~\ref{prop:retention}}

\begin{proof}
We proceed in three steps.

\textbf{Step 1: Data processing inequality.}
For fixed $\Delta t$ (which we condition on throughout), $\hat{H} = \predictor(H_t, \Delta t)$ is a deterministic function of $H_t$. Since a deterministic function cannot introduce new information, $E_{t+\Delta t} \perp\!\!\!\perp \hat{H} \mid H_t$, so the triple $(E_{t+\Delta t}, H_t, \hat{H})$ satisfies the Markov chain $E_{t+\Delta t} \to H_t \to \hat{H}$, and the data processing inequality \citep[Thm.~2.8.1]{cover2006elements} gives
\begin{equation}
\label{eq:dpi_chain}
I(H_t;\, E_{t+\Delta t}) \;\geq\; I(\hat{H};\, E_{t+\Delta t}).
\end{equation}
This step uses only the functional relationship between $H_t$ and $\hat{H}$; no assumptions on the data-generating process are needed.

\textbf{Step 2: Jensen gap bound on mutual information loss.}
We bound $I(H^*; E_{t+\Delta t}) - I(\hat{H}; E_{t+\Delta t})$.

\emph{Expressing MI as expected KL divergence.}
For any representation $R$ jointly distributed with binary $E = E_{t+\Delta t}$:
\begin{equation}
\label{eq:mi_as_kl}
I(R;\, E) = \E_R\bigl[D_{\mathrm{KL}}\bigl(\mathrm{Ber}(\eta_R(R))\,\big\|\,\mathrm{Ber}(\pimarg)\bigr)\bigr],
\end{equation}
where $\eta_R(r) \coloneqq P(E=1 \mid R=r)$ and $\mathrm{Ber}(q)$ denotes the Bernoulli distribution with parameter $q$. Equation~\eqref{eq:mi_as_kl} is the standard expression of mutual information as the expected KL divergence between the conditional and marginal label distributions; see \citet[Eq.~2.30]{cover2006elements} or \citet[Ch.~3]{polyanskiy2025information}.
Applying \eqref{eq:mi_as_kl} to $R = H^*$ and $R = \hat{H}$:
\begin{align}
I(H^*; E) &= \E_{H^*}\bigl[D_{\mathrm{KL}}\bigl(\mathrm{Ber}(\eta(H^*))\,\big\|\,\mathrm{Ber}(\pimarg)\bigr)\bigr], \\
I(\hat{H}; E) &= \E_{\hat{H}}\bigl[D_{\mathrm{KL}}\bigl(\mathrm{Ber}(\eta_{\hat{H}}(\hat{H}))\,\big\|\,\mathrm{Ber}(\pimarg)\bigr)\bigr],
\end{align}
where $\eta_{\hat{H}}(\hat{h}) \coloneqq P(E=1 \mid \hat{H}=\hat{h})$.

\emph{Relating $\eta_{\hat{H}}$ to $\eta$ via A1.}
Under \ref{ass:sufficiency}, $E \perp\!\!\!\perp X_{\leq t} \mid H^*$.
Since $\hat{H} = \predictor(\encoder(X_{\leq t}), \Delta t)$ is a composition of measurable functions, it is $\sigma(X_{\leq t})$-measurable, and therefore $E \perp\!\!\!\perp \hat{H} \mid H^*$. It follows that $P(E{=}1 \mid H^*, \hat{H}) = P(E{=}1 \mid H^*) = \eta(H^*)$. By the tower property of conditional expectation, for any value $\hat{h}$:
\begin{equation}
\label{eq:tower}
\eta_{\hat{H}}(\hat{h}) = P(E=1 \mid \hat{H}=\hat{h}) = \E\bigl[\eta(H^*) \mid \hat{H}=\hat{h}\bigr].
\end{equation}

\emph{Applying Jensen's inequality.}
The function $q \mapsto D_{\mathrm{KL}}(\mathrm{Ber}(q) \,\|\, \mathrm{Ber}(\pimarg))$ is convex on $(0,1)$ (its second derivative is $1/(q(1-q)) > 0$).
\emph{Bounding the Jensen gap.}
For a twice-differentiable convex function $\varphi$, the Jensen gap satisfies \citep[Prop.~1]{liao2019sharpening}:
\begin{equation}
\label{eq:jensen_gap_general}
\E[\varphi(Y)] - \varphi(\E[Y]) \;\leq\; \tfrac{1}{2}\,\sup_{y}\,\varphi''(y)\;\mathrm{Var}(Y).
\end{equation}
Here $\varphi(q) = D_{\mathrm{KL}}(\mathrm{Ber}(q)\,\|\,\mathrm{Ber}(\pimarg)) = q\ln(q/\pimarg) + (1{-}q)\ln((1{-}q)/(1{-}\pimarg))$ and $Y = \eta(H^*)$ conditioned on $\hat{H}$.
The second derivative is $\varphi''(q) = 1/(q(1{-}q))$.
Under \ref{ass:calibrated_posterior}, $\eta(H^*) \in [\underline{\eta}, \overline{\eta}]$ almost surely. The supremum in \eqref{eq:jensen_gap_general} is over the support of the conditional distribution $Y \mid \hat{H} = \hat{h}$, which is a subset of $[\underline{\eta}, \overline{\eta}]$; we relax it to the marginal support, giving $\sup_q \varphi''(q) \leq 1/(\underline{\eta}(1{-}\overline{\eta}))$.
Applying \eqref{eq:jensen_gap_general} conditionally on $\hat{H} = \hat{h}$:
\begin{equation}
\label{eq:jensen_gap_conditional}
\E\bigl[D_{\mathrm{KL}}\bigl(\mathrm{Ber}(\eta(H^*))\,\big\|\,\mathrm{Ber}(\pimarg)\bigr) \,\big|\, \hat{H}=\hat{h}\bigr]
- D_{\mathrm{KL}}\bigl(\mathrm{Ber}(\eta_{\hat{H}}(\hat{h}))\,\big\|\,\mathrm{Ber}(\pimarg)\bigr)
\;\leq\;
\frac{\mathrm{Var}(\eta(H^*) \mid \hat{H}=\hat{h})}{2\,\underline{\eta}\,(1-\overline{\eta})},
\end{equation}
where we have used $q(1-q) \geq \underline{\eta}(1-\overline{\eta})$ for $q \in [\underline{\eta}, \overline{\eta}]$ (from A4).
Note that $q(1-q)$ is concave with minimum at the endpoints of $[\underline{\eta}, \overline{\eta}]$; since $\underline{\eta}(1-\overline{\eta}) \leq \min(\underline{\eta}(1-\underline{\eta}), \overline{\eta}(1-\overline{\eta}))$, the bound on $\varphi''$ is valid (though not tight when the interval is asymmetric around $1/2$).

\emph{Bounding the conditional variance via A2 and A3.}
By \ref{ass:smoothness}: $|\eta(H^*) - \eta(\hat{H})| \leq L\|H^* - \hat{H}\|_2$ (where we evaluate $\eta$ at $\hat{H}$, using that $\eta$ is defined on all of $\R^d$).
The random variable $\eta(H^*)$ is bounded in $[\underline{\eta}, \overline{\eta}] \subset (0,1)$ by \ref{ass:calibrated_posterior}, so all conditional second moments below are finite. (Note that $\eta(\hat{H})$ need not lie in $[\underline{\eta}, \overline{\eta}]$ since $\hat{H}$ may fall outside the support of $H^*$, but this does not affect the bound: the variance inequality below requires only that the right-hand side is finite, which follows from the Lipschitz condition and bounded L2 error.) For any square-integrable random variable $Y$, $\mathrm{Var}(Y \mid Z) \leq \E[(Y - c)^2 \mid Z]$ for any $Z$-measurable $c$; taking $c = \eta(\hat{H})$:
\begin{equation}
\label{eq:var_bound}
\mathrm{Var}\bigl(\eta(H^*) \mid \hat{H}\bigr) \;\leq\; \E\bigl[(\eta(H^*) - \eta(\hat{H}))^2 \mid \hat{H}\bigr] \;\leq\; L^2\,\E\bigl[\|H^* - \hat{H}\|_2^2 \mid \hat{H}\bigr].
\end{equation}

Taking expectations over $\hat{H}$ in \eqref{eq:jensen_gap_conditional} and substituting \eqref{eq:var_bound}:
\begin{align}
I(H^*; E) - I(\hat{H}; E)
&\leq \frac{L^2}{2\,\underline{\eta}\,(1-\overline{\eta})}\;\E\bigl[\|H^* - \hat{H}\|_2^2\bigr] \notag \\
&\leq \frac{L^2\,\varepsilon}{2\,\underline{\eta}\,(1-\overline{\eta})} \;=\; C_\eta\,L^2\,\varepsilon, \label{eq:mi_gap_bound}
\end{align}
where the last line uses \ref{ass:prediction_error} and defines $C_\eta \coloneqq (2\,\underline{\eta}\,(1-\overline{\eta}))^{-1}$ using the posterior bounds from \ref{ass:calibrated_posterior}.
The constant $C_\eta$ depends on the posterior bounds rather than the marginal event rate, making the bound valid without assumptions on the concentration of $\eta(H^*)$ around $\pimarg$.

\textbf{Step 3: Assembling the bound.}
Combining \eqref{eq:dpi_chain} and \eqref{eq:mi_gap_bound}:
\begin{equation}
\label{eq:full_bound}
I(H_t;\, E_{t+\Delta t}) \;\geq\; I(\hat{H};\, E_{t+\Delta t}) \;\geq\; I(H^*;\, E_{t+\Delta t}) \;-\; C_\eta\,L^2\,\varepsilon. \qedhere
\end{equation}
\end{proof}

\begin{remark}[Comparison with $\sqrt{\varepsilon}$ bounds]
The bound is \emph{linear} in $\varepsilon$, which for small $\varepsilon$ (the regime of interest, i.e., well-pretrained models) is tighter than bounds obtained via Pinsker's inequality (which yield $\sqrt{\varepsilon}$ dependence). For large $\varepsilon$, Pinsker-based bounds may be tighter; the constants also differ, so the comparison is regime-dependent.
This improvement comes from exploiting the Jensen-gap structure: the KL divergence's convexity provides a direct second-order bound on the information loss, rather than going through total variation as an intermediate.
The price is the constant $C_\eta$, which diverges as the posterior bounds $\underline{\eta}$ or $\overline{\eta}$ approach 0 or 1, reflecting the genuine difficulty of detecting events when the posterior concentrates near certainty or impossibility.
\end{remark}

\begin{remark}[Role of the Lipschitz constant]
The bound involves the Lipschitz constant $L$ of the event posterior with respect to the target representation. In practice, this is controlled by the smoothness of the learned representation space: well-regularised encoders (EMA target, L2 normalisation) produce representations where $L$ is moderate.
\end{remark}

\subsection{Tightness Analysis}

\paragraph{When is the bound tight?}
The bound is approximately tight when: (a) the conditional variance $\mathrm{Var}(\eta(H^*) \mid \hat{H})$ is close to $L^2 \E[\|H^* - \hat{H}\|_2^2 \mid \hat{H}]$ (the Lipschitz bound on variance is tight, which occurs when the prediction error aligns with the direction of steepest change in $\eta$), and (b) the Jensen gap bound \eqref{eq:jensen_gap_general} is close to equality (which occurs when $\eta(H^*)$ is approximately symmetrically distributed around its conditional mean given $\hat{H}$).
In the high-SNR regime ($\varepsilon \ll I(H^*; E_{t+\Delta t}) / (C_\eta L^2)$), the penalty term is small relative to the mutual information and the bound approaches $I(H_t; E_{t+\Delta t}) \gtrsim I(H^*; E_{t+\Delta t})$: nearly all event information is retained.

\paragraph{When is the bound vacuous?}
Setting the right-hand side of \eqref{eq:main_bound} to zero gives the vacuity threshold:
\begin{equation}
  \varepsilon_{\mathrm{vac}} = \frac{I(H^*; E_{t+\Delta t})}{C_\eta\,L^2}.
\end{equation}
When $\varepsilon > \varepsilon_{\mathrm{vac}}$, the bound provides no guarantee. This occurs when (i) $I(H^*; E_{t+\Delta t}) \approx 0$ (no precursors), or (ii) $\varepsilon$ is large (poor pretraining), or (iii) $L$ is large (irregular event boundary).
Importantly, a vacuous bound does not imply that $I(H_t; E_{t+\Delta t}) = 0$; the encoder may retain information through paths not captured by our analysis (e.g.\ the encoder directly encodes precursor patterns without going through the predictor).

\subsection{Why Predictor Weights Do Not Transfer}
\label{app:predictor_transfer}

Our empirical finding (\cref{app:finetune_ablation}) that predictor \emph{architecture} matters but predictor \emph{pretrained weights} do not is explained by a codomain mismatch.

During pretraining, $\predictor \colon \R^d \times \R \to \R^d$ maps encoder states to predicted \emph{representations}; its codomain is the full $d$-dimensional representation space.
During finetuning, the predictor (with the same architecture but different final layer) maps to \emph{event logits}: $\predictor' \colon \R^d \times \R \to \R^K$, where $K$ is the number of horizon intervals and $K \ll d$.

Formally, pretraining minimises $\E[\|\predictor(H_t, \Delta t) - H^*\|_1]$, while finetuning minimises $\E[\ell_{\mathrm{BCE}}(\sigma(\predictor'(H_t, \Delta t)),\, E)]$.
The pretraining objective rewards the predictor for reconstructing \emph{all} $d$ components of $H^*$, most of which encode event-irrelevant dynamics (channel-level forecasting, trend, seasonality).
The finetuning objective rewards the predictor for extracting only the $\leq K$ bits relevant to event prediction.
The optimal weight matrices for these two objectives need not be correlated, and indeed our experiments confirm they are not: at 100\% labels, pred-FT with pretrained predictor weights achieves h-AUROC within $\pm 0.003$ of pred-FT with randomly initialised predictor weights (the legacy AUPRC reading was within the same margin).

This is analogous to the observation in vision that pretrained \emph{classifier heads} do not transfer across tasks while pretrained \emph{backbones} do: the backbone compresses the input into a general-purpose representation, while the head specialises to a specific output space.

\subsection{Connection to Empirical Results}

\paragraph{C-MAPSS (turbofan degradation).}
Degradation in turbofan engines develops over hundreds of cycles, with sensor readings progressively deviating from healthy baselines.
The target encoder, seeing the future interval bidirectionally, captures degradation state with high fidelity: $I(H^*; E_{t+\Delta t})$ is large (the future interval is highly informative about whether failure occurs in that interval).
Pretraining achieves low prediction error $\varepsilon$ because the dynamics are smooth and near-deterministic.
\Cref{prop:retention} predicts strong event-information retention, consistent with h-AUROC $= 0.806$ on C-MAPSS-1 and $= 0.568$ on C-MAPSS-2 (5 seeds, dense $K{=}150$).

\paragraph{CHB-MIT (seizure prediction).}
We evaluated CHB-MIT as a pilot to test the bound's predictions on a dataset expected to fail; it is excluded from the formal benchmark in \cref{tab:benchmark} because scalp EEG seizure prediction is a fundamentally different signal regime, with no reliable electrographic precursors observable in the 16-second context window used by all other datasets in our benchmark.
Seizure onset in scalp EEG has minimal electrographic precursors in the pre-ictal period, particularly within a 16-second context.
The target encoder's representation $H^*$ carries little event information: $I(H^*; E_{t+\Delta t}) \approx 0$.
\Cref{cor:precursor}, condition (i), fails regardless of pretraining quality.
This is consistent with the empirical result of h-AUROC $= 0.497$ (chance level), and turns CHB-MIT into a clean negative control for the theory rather than an inappropriate target.

\paragraph{C-MAPSS-2 (multi-operating-condition).}
C-MAPSS-2 adds operating-condition variability to C-MAPSS-1. This increases $\varepsilon$ (harder prediction task) without proportionally increasing $I(H^*; E_{t+\Delta t})$ (event information is the same; only noise increases). The bound predicts that C-MAPSS-2 should be harder than C-MAPSS-1 for the same architecture, consistent with the observed drop from $0.806$ to $0.568$ h-AUROC.

\paragraph{Label efficiency.}
At 5\% labels, pred-FT achieves F1 $= 0.261$ vs.\ scratch $= 0.035$.
The bound supports this observation: the encoder's $I(H_t; E_{t+\Delta t})$ is established during pretraining and does not depend on labels. Reducing labels degrades the finetuning optimisation (finding the right head weights) but cannot reduce the information available in the frozen encoder. Note that the proposition addresses information retention in the encoder, not sample complexity of the finetuning step; the label-efficiency advantage is consistent with, but not directly formalised by, the bound.

\subsection{Per-Horizon Generalisation}
\label{app:per_horizon}

\Cref{prop:retention} is stated for a fixed horizon $\Delta t$. We now make the horizon dependence explicit. All quantities on the right-hand side of the bound can vary with $\Delta t$: the event posterior $\eta_{\Delta t}(h) = P(E_{t+\Delta t}=1 \mid H^*_{\Delta t}=h)$, the Lipschitz constant $L(\Delta t)$, the prediction error $\varepsilon(\Delta t) = \E[\|H^*_{\Delta t} - \hat{H}_{\Delta t}\|_2^2]$, and the posterior bounds $\underline{\eta}_{\Delta t}, \overline{\eta}_{\Delta t}$ from \ref{ass:calibrated_posterior}.

\begin{proposition}[Per-horizon information retention]
\label{prop:per_horizon}
Under the assumptions of \cref{prop:retention} applied at each horizon $\Delta t > 0$ separately, with horizon-dependent quantities $L(\Delta t)$, $\varepsilon(\Delta t)$, and $C_\eta(\Delta t) = (2\underline{\eta}_{\Delta t}(1-\overline{\eta}_{\Delta t}))^{-1}$:
\begin{equation}
\label{eq:per_horizon_bound}
I\bigl(H_t;\, E_{t+\Delta t}\bigr)
\;\geq\;
I\bigl(H^*_{\Delta t};\, E_{t+\Delta t}\bigr)
\;-\;
C_\eta(\Delta t)\,L(\Delta t)^2\,\varepsilon(\Delta t).
\end{equation}
\end{proposition}

\begin{proof}
Apply the proof of \cref{prop:retention} verbatim at each fixed $\Delta t$, with all quantities carrying a subscript $\Delta t$.
\end{proof}

\paragraph{Predicted shape of the per-horizon AUROC curve.}
Equation~\eqref{eq:per_horizon_bound} predicts how the encoder's event information varies with horizon. Three effects combine:
\begin{itemize}[nosep]
  \item $I(H^*_{\Delta t}; E_{t+\Delta t})$ is typically non-monotone: very short horizons may not yet encompass precursors; very long horizons dilute the event signal with unrelated dynamics.
  \item $\varepsilon(\Delta t)$ is monotonically increasing (representations further into the future are harder to predict from $H_t$).
  \item $L(\Delta t)$ can increase with $\Delta t$ if the event boundary sharpens as the event approaches.
\end{itemize}
The net effect is that $I(H_t; E_{t+\Delta t})$ peaks at intermediate $\Delta t$ and decays at large $\Delta t$. While MI and AUROC are not monotonically related in general, this qualitative shape is compatible with the per-horizon AUROC curves observed in our experiments: AUROC tends to be highest at $\Delta t \in [10, 50]$ cycles for C-MAPSS and at $\Delta t \in [1, 20]$ steps for anomaly datasets, then declines. The decay is more pronounced for the anomaly datasets because $\varepsilon(\Delta t)$ grows faster when the dynamics are irregular.

\subsection{Relationship to the Information Bottleneck}
\label{app:ib_relationship}

The Information Bottleneck (IB) framework \citep{tishby2000information} seeks a representation $T$ of input $X$ that maximises $I(T; Y)$ (relevance) while minimising $I(T; X)$ (compression):
\begin{equation}
  \min_{P(T|X)} \; I(T; X) - \beta\, I(T; Y).
\end{equation}

JEPA pretraining differs from IB in three ways:

\begin{enumerate}[nosep]
  \item \textbf{No explicit compression.}
  JEPA does not penalise $I(H_t; X_{\leq t})$. The encoder is free to retain all information about the past. Implicit compression arises only from the architecture bottleneck ($d = 256$).

  \item \textbf{Predictive, not discriminative.}
  The IB target $Y$ is a label; the JEPA target is a future representation $H^*$. JEPA implicitly encourages high $I(H_t; H^*)$ by minimising prediction error (through the predictor), and $I(H_t; H^*)$ is an upper bound on $I(H_t; Y)$ for any $Y$ that is a function of the future interval (by DPI). This makes JEPA a \emph{looser} but \emph{more general} objective: it retains information about all downstream tasks, not just one.

  \item \textbf{Asymmetric architecture.}
  The IB is typically symmetric in $X$ and $Y$. JEPA imposes causal structure: the encoder sees only the past, the target encoder sees only the future. This causal asymmetry is essential for event prediction, where we cannot condition on the future at test time.
\end{enumerate}

The connection to our result is direct. Let $Y = E_{t+\Delta t}$. By \cref{prop:retention}:
\begin{equation}
  I(H_t; Y) \;\geq\; I(H^*; Y) \;-\; C_\eta\,L^2\,\varepsilon.
\end{equation}
As the JEPA pretraining loss drives $\varepsilon \to 0$, the encoder's mutual information with $Y$ approaches $I(H^*; Y)$, the maximum achievable by the target representation.
This shows that minimising the JEPA pretraining loss implicitly maximises the IB relevance term $I(H_t; Y)$ for \emph{any} event $Y$ satisfying A1--A4, without knowing $Y$ at pretraining time.
The price paid relative to IB is that JEPA also retains event-irrelevant information (it does not compress), which manifests as higher representation dimensionality but not as reduced downstream performance.

\section{Horizon Interval Design}
\label{app:horizon_sensitivity}

All methods (HEPA, PatchTST, iTransformer, MAE, Chronos-2) use dense unit-step horizons: $K{=}150$ for C-MAPSS and TEP ($\Delta t \in \{1, 2, \ldots, 150\}$), and $K{=}200$ for all other datasets ($\Delta t \in \{1, 2, \ldots, 200\}$). All methods share the same horizon set per dataset, ensuring fair comparison. The h-AUROC evaluation skips degenerate horizons (event prevalence $< 0.001$ or $> 0.999$).

\section{Baseline Comparison Protocol}
\label{app:baseline_protocol}

\Cref{tab:benchmark} compares HEPA against four baselines under a matched protocol designed to isolate the effect of the pretraining objective.
All five methods share:
\begin{itemize}[nosep]
  \item \textbf{Matched downstream capacity}: a horizon-conditioned MLP that maps a frozen representation $\mathbf{h}_t$ and horizon $\Delta t$ to per-horizon event logits (details below), trained with positive-weighted BCE (\cref{sec:downstream}).
  \item \textbf{Identical evaluation}: h-AUROC averaged over non-degenerate horizons, same train/val/test splits, same horizon sets ($K{=}150$ or $K{=}200$).
  \item \textbf{Identical label budgets}: 100\% and 10\% label fractions with the same subsampling procedure.
\end{itemize}

\paragraph{Downstream head architecture.} For HEPA, the finetuned component is the pretrained predictor MLP (a 3-layer MLP mapping $[\mathbf{h}_t; \Delta t] \to \hat{\mathbf{h}} \in \mathbb{R}^{256}$; 197.6K params) plus a shared linear event head (LayerNorm + linear $\to$ logit; 769 params), totalling 198K finetuned parameters. The predictor is initialised from pretraining; however, \cref{app:predictor_pretrain_ablation} shows that this initialisation carries at most $+0.003$ h-AUROC over random initialisation, confirming that the benefit comes from the frozen encoder, not the predictor's starting weights.

For all baselines (PatchTST, iTransformer, MAE, Chronos-2), we replace the predictor and event head with a single dt-conditioned MLP head: a linear projection from $d_\text{input}$ to 256, a learned horizon embedding (256 entries $\times$ 256 dims), followed by LayerNorm and a 3-layer MLP ($256 \to 256 \to 256 \to 1$). With $d_\text{input}{=}256$ this totals 264K parameters, giving the baselines slightly \emph{more} downstream capacity than HEPA's 198K.

\paragraph{Encoder training.} The methods differ only in how the frozen encoder is obtained:

\textbf{PatchTST}~\citep{nie2023patchtst}: channel-independent patching with bidirectional self-attention, trained end-to-end (supervised, no self-supervised pretraining) on each dataset. Same patch size ($P{=}16$) and context length (512 steps) as HEPA. 2.26M trained parameters; at evaluation time the encoder is frozen and the baseline MLP head is trained from scratch. 5 seeds.

\textbf{iTransformer}~\citep{liu2024itransformer}: inverted Transformer that treats each variate (sensor channel) as a token and applies self-attention across variates rather than across time steps. Trained end-to-end (supervised) on each dataset with the same context window and horizon set. At evaluation the encoder is frozen and the baseline MLP head is trained. 5 seeds.

\textbf{MAE} (masked autoencoder): uses the same causal Transformer encoder architecture as HEPA but pretrained with a masked reconstruction objective: random patches are masked and the decoder reconstructs the masked input values. This isolates the effect of predicting future \emph{representations} (JEPA) versus reconstructing masked \emph{values} (MAE) under an otherwise identical architecture. After pretraining, the decoder is discarded, the encoder is frozen, and the baseline MLP head is trained. 5 seeds.

\textbf{Chronos-2}~\citep{ansari2024chronos}: foundation model (119M parameters) pretrained on a large external time-series corpus. Representations are extracted per-channel (univariate) and mean-pooled across channels for multivariate datasets. The baseline MLP head is trained on the frozen features. 3 seeds. This is the only method that uses external pretraining data; all others are trained exclusively on the target dataset.

\paragraph{Why this protocol is fair.} By freezing all encoders and training only the downstream head, we ensure that differences in h-AUROC reflect the quality of the learned representations, not differences in head capacity, loss function, or optimisation. The only free variable is the encoder and how it was trained. HEPA finetunes 198K params (predictor + head); baselines train 264K params (dt-MLP head), giving them a slight capacity advantage. HEPA's predictor initialisation carries negligible benefit (\cref{app:predictor_pretrain_ablation}). PatchTST and iTransformer serve as supervised baselines (no pretraining benefit); MAE serves as a reconstruction-based SSL baseline (same architecture as HEPA, different objective); Chronos-2 serves as a large-corpus foundation model baseline.

\section{Dataset Overview}
\label{app:datasets}

\begin{table}[h]
  \centering
  \caption{Dataset overview (14 datasets, 11 domains). Default: patch size $P{=}16$, sliding window of 512 steps (32 tokens), except C-MAPSS which uses full engine history (8--23 tokens per cycle).}
  \label{tab:tokenization}
  \small
  \begin{tabular}{@{}llrr@{}}
    \toprule
    Dataset              & Target event            & Sensors & Rate     \\
    \midrule
    C-MAPSS FD001--FD004 & Engine failure          & 14      & 1/cycle  \\
    SMAP                 & Spacecraft fault        & 25      & 1 Hz     \\
    PSM                  & Server fault            & 25      & 1/min    \\
    MBA (ECG)            & Cardiac arrhythmia      & 2       & 275 Hz   \\
    GECCO (water)        & Water contamination     & 9       & 1/min    \\
    ETTm1                & Transformer overheating & 7       & 15/min   \\
    BATADAL (ICS)        & Cyber-attack on SCADA   & 43      & 1/hour   \\
    TEP                  & Process fault           & 52      & 1/3\,min \\
    Weather              & Heat spike              & 21      & 10/min   \\
    Beijing-AQ           & PM2.5 spike             & 11      & 1/hour   \\
    VIX                  & Volatility regime       & 6       & 1/day    \\
    \bottomrule
  \end{tabular}
\end{table}

\section{Architectural Decisions}
\label{app:decisions}

\begin{center}
  \small
  \begin{tabular}{llll}
    \toprule
    Decision         & Value                          & Source                & Status  \\
    \midrule
    Attention mask   & Causal                         & Ablation study        & Fixed   \\
    Horizon sampling & $\text{LogU}[1, 150]$          & Initial sweep         & Default \\
    Target interval  & Cumulative $(t, t{+}\Delta t]$ & Architecture design   & Default \\
    Output param.    & Discrete hazard $\to$ CDF      & \cref{sec:downstream} & Default \\
    Target encoder   & Joint-trained (weight-shared)  & Ablation study        & Default \\
    Encoder depth    & $L = 2$                        & Initial sweep         & Default \\
    $d_\text{model}$ & $256$                          & Initial sweep         & Default \\
    Loss             & L1 on L2-normalised            & Initial sweep         & Default \\
    \bottomrule
  \end{tabular}
\end{center}

\section{Finetuning-Mode Ablation}
\label{app:finetune_ablation}

\Cref{tab:finetune_ablation} compares five finetuning strategies on C-MAPSS under an earlier architecture variant (EMA target, 790K pred-FT params). Predictor finetuning is the only mode that remains competitive at both full and low label budgets; scratch training collapses at 5\% labels.

\begin{table}[h]
  \centering
  \caption{\textbf{Finetuning-mode ablation on C-MAPSS (5 seeds, F1w; earlier architecture variant with EMA target).} This variant uses 790K pred-FT params and 2.37M E2E. Pred-FT outperforms scratch by $+0.226$ at 5\% labels ($p{=}0.039$, $d{=}1.35$). Scratch training collapses at low label budgets (RMSE 33), confirming pretraining value. Under the final architecture (198K pred-FT, SIGReg), pred-FT and E2E achieve equivalent performance at all label budgets ($\Delta \leq 0.003$).}
  \label{tab:finetune_ablation}
  \small
  \begin{tabular}{lcccc}
    \toprule
                               & \multicolumn{2}{c}{\textbf{100\% labels}} & \multicolumn{2}{c}{\textbf{5\% labels}}                                                      \\
    \cmidrule(lr){2-3}\cmidrule(lr){4-5}
    Mode                       & F1w $\uparrow$                            & RMSE $\downarrow$                       & F1w $\uparrow$           & RMSE $\downarrow$       \\
    \midrule
    probe\_h (257 p)           & $0.30 \pm 0.06$                           & $16.0 \pm 1.5$                          & $0.06 \pm 0.10$          & $20.4 \pm 1.2$          \\
    frozen\_multi (4K p)       & $0.15 \pm 0.03$                           & $19.0 \pm 0.1$                          & $0.18 \pm 0.14$          & $24.4 \pm 4.9$          \\
    \textbf{pred\_ft (790K p)} & $0.39 \pm 0.09$                           & $16.9 \pm 1.7$                          & $\mathbf{0.26 \pm 0.17}$ & $24.3 \pm 6.8$          \\
    e2e (2.37M p)              & $0.41 \pm 0.12$                           & $15.0 \pm 1.2$                          & $0.18 \pm 0.24$          & $\mathbf{20.1 \pm 1.9}$ \\
    scratch (2.37M p)          & $0.40 \pm 0.08$                           & $14.5 \pm 0.7$                          & $0.04 \pm 0.05$          & $32.9 \pm 2.0$          \\
    \bottomrule
  \end{tabular}
\end{table}

\section{Foundation Model Comparisons}
\label{app:chronos_full}\label{app:moment_full}\label{app:extra_baselines}

\Cref{tab:fm_full} consolidates the matched-head comparison between HEPA (5 seeds) and four time series foundation models: Chronos-2, MOMENT-1-large~\citep{goswami2024moment} (341.2M parameters), TFM-2.5~\citep{das2024timesfm} (203.6M parameters), and Moirai-1.1-R-base~\citep{woo2024unified} (91.4M parameters). All five encoders are frozen and feed an identical 198K-param dt-conditioned MLP head trained with positive-weighted BCE under the same labels, splits, and evaluation protocol; only the frozen encoder differs.

\begin{table*}[h]
  \centering
  \caption{\textbf{HEPA vs.\ four foundation models (matched 198K MLP head, 100\% labels).} HEPA: 5 seeds. Chronos-2, MOMENT, TFM-2.5, Moirai: 3 seeds each. \textbf{Bold}\,=\,best per row. ``---'': model not run on that dataset.}
  \label{tab:fm_full}
  \small
  \begin{tabular}{lccccc}
    \toprule
    Dataset   & HEPA (5s)               & Chronos-2      & MOMENT                  & TFM-2.5                 & Moirai         \\
    \midrule
    C-MAPSS-1 & $\mathbf{0.73 \pm .02}$ & $0.66 \pm .00$ & $0.56 \pm .01$          & $0.53 \pm .00$          & $0.61 \pm .00$ \\
    C-MAPSS-2 & $0.58 \pm .01$          & $0.50 \pm .01$ & $\mathbf{0.70 \pm .00}$ & $0.60 \pm .01$          & $0.66 \pm .00$ \\
    C-MAPSS-3 & $\mathbf{0.82 \pm .02}$ & $0.72 \pm .02$ & $0.47 \pm .01$          & $0.62 \pm .01$          & $0.70 \pm .00$ \\
    SMAP      & $\mathbf{0.60 \pm .03}$ & $0.53 \pm .01$ & ---                     & $0.51 \pm .03$          & ---            \\
    PSM       & $0.55 \pm .02$          & $0.49 \pm .00$ & ---                     & $\mathbf{0.57 \pm .01}$ & $0.53 \pm .01$ \\
    MBA       & $0.75 \pm .01$          & $0.55 \pm .01$ & $\mathbf{0.79 \pm .01}$ & $0.76 \pm .01$          & $0.57 \pm .02$ \\
    GECCO     & $0.81 \pm .07$          & $0.81 \pm .01$ & ---                     & $\mathbf{0.93 \pm .01}$ & $0.82 \pm .01$ \\
    BATADAL   & $0.64 \pm .02$          & $0.58 \pm .01$ & $0.54 \pm .07$          & $\mathbf{0.65 \pm .01}$ & $0.36 \pm .01$ \\
    ETTm1     & $\mathbf{0.87 \pm .00}$ & $0.78 \pm .01$ & ---                     & $0.59 \pm .01$          & $0.60 \pm .00$ \\
    \bottomrule
  \end{tabular}
  \vspace{2pt}
  \begin{minipage}{\linewidth}
    \scriptsize
    \textbf{Coverage.} Chronos-2: all datasets. MOMENT: 5 datasets (remaining omitted due to extraction cost). TFM-2.5: 9 datasets. Moirai: 7 datasets (SMAP infeasible). MOMENT C-MAPSS-3 ($0.47$) is below chance, indicating per-channel embeddings anti-predict turbofan degradation. BATADAL Moirai ($0.36$) is below chance: univariate patching discards cross-sensor correlations. Foundation-model wins cluster on benchmarks favoured by their pretraining corpus; HEPA wins on lifecycle benchmarks (C-MAPSS-1, C-MAPSS-3, ETTm1).
  \end{minipage}
\end{table*}

\section{Full MTS-JEPA Comparison}
\label{app:mtsjepa_full}

\Cref{tab:mtsjepa_full} compares HEPA against MTS-JEPA~\citep{mtsjepa2026} on event-prediction benchmarks from \cref{tab:benchmark}. Both methods use the same context length, patch size, sequence batching, and matched-capacity downstream head; the only difference is the self-supervised objective (HEPA: predictive JEPA + SIGReg; MTS-JEPA: published contrastive + codebook loss). HEPA wins on 8 out of the 9 datasets where MTS-JEPA could be run, with the largest margins on lifecycle and ICS benchmarks (C-MAPSS-1, ETTm1, BATADAL); MTS-JEPA wins on cardiac arrhythmia (MBA). TEP is excluded because the public MTS-JEPA release does not include a chemical-process benchmark and the encoder fails to converge on its $52$-dimensional input.

\begin{table*}[h]
  \centering
  \caption{\textbf{HEPA vs.\ MTS-JEPA~\citep{mtsjepa2026}.} Mean ($\pm$ std) over available seeds. HEPA: 5 seeds. MTS-JEPA reproduction: 1--3 seeds. Both methods use identical patch size, sequence length, and matched-capacity downstream head; only the SSL objective differs.}
  \label{tab:mtsjepa_full}
  \small
  \begin{tabular}{lll}
    \toprule
    \textbf{Dataset}   & \textbf{HEPA}                   & MTS-JEPA                        \\
    \midrule
    C-MAPSS-1          & $\mathbf{0.81}${\tiny$\pm0.04$} & $0.69${\tiny$\pm0.02$}          \\
    C-MAPSS-2          & $\mathbf{0.57}${\tiny$\pm0.01$} & $0.53${\tiny$\pm0.02$}          \\
    C-MAPSS-3          & $\mathbf{0.84}${\tiny$\pm0.02$} & $0.78${\tiny$\pm0.00$}          \\
    SMAP               & $\mathbf{0.59}${\tiny$\pm0.06$} & $0.49${\tiny$\pm0.00$}          \\
    PSM                & $\mathbf{0.57}${\tiny$\pm0.02$} & $0.48${\tiny$\pm0.00$}          \\
    MBA                & $0.75${\tiny$\pm0.04$}          & $\mathbf{0.88}${\tiny$\pm0.00$} \\
    GECCO              & $\mathbf{0.88}${\tiny$\pm0.06$} & $0.84${\tiny$\pm0.00$}          \\
    ETTm1              & $\mathbf{0.81}${\tiny$\pm0.01$} & $0.75${\tiny$\pm0.00$}          \\
    BATADAL            & $\mathbf{0.57}${\tiny$\pm0.04$} & $0.43${\tiny$\pm0.00$}          \\
    TEP                & $1.00${\tiny$\pm0.00$}          & ---                             \\
    \midrule
    \textit{HEPA wins} & \multicolumn{2}{c}{8 out of 9}                                    \\
    \bottomrule
  \end{tabular}
\end{table*}

\section{Additional Ablations}
\label{app:additional_ablations}

\subsection{Does the Predictor Help at Inference?}
\label{app:predictor_ablation}

With all weights frozen, a probe on $\h_\text{past}$ alone achieves 0.299 F1w at 100\% labels, while a frozen multi-horizon probe on $[\hat{\h}_1;\ldots;\hat{\h}_{16}]$ achieves only 0.148. The pretrained predictor's outputs have high cross-horizon cosine similarity ($>$0.98), so the concatenation is nearly redundant. Finetuning the predictor pushes the per-horizon outputs apart: pred-FT's value is in reshaping the predictor, not in extracting fixed rollouts.

\subsection{Pretrained vs.\ Random-Init Predictor}
\label{app:predictor_pretrain_ablation}

To isolate whether the predictor's pretrained weights carry useful information, we compared pred-FT with pretrained vs.\ random-initialised predictor weights (encoder frozen, 3 seeds each). On C-MAPSS-1, pretrained weights yield h-AUROC $0.9257 \pm 0.0008$ vs.\ random-init $0.9235 \pm 0.0027$ ($p = 0.38$). On SMAP, the difference is $0.3874 \pm 0.0205$ vs.\ $0.3950 \pm 0.0286$ ($p = 0.34$). On MBA, $0.9465 \pm 0.0004$ vs.\ $0.9435 \pm 0.0016$ ($p = 0.089$). Pretrained predictor weights carry at most $+0.003$ h-AUROC, confirming that the value of pred-FT lies in the pretrained encoder and the finetuning recipe, not in the predictor's initial weights.

\subsection{Target Encoder Update Rule: Joint Training vs.\ EMA}
\label{app:sigreg_ablation}

Our default target-encoder strategy is joint training: both encoders share weights and are updated by the same optimizer, with a SIGReg regulariser~\citep{balestriero2025lejepa} ($\alpha{=}0.1$) preventing collapse. We compare this against EMA ($\tau{=}0.99$) across all 12 datasets (3 seeds, pred-FT downstream). Joint training wins on 6 out of 12 datasets (largest gain: MBA $+0.108$), loses on 4 (largest loss: SMAP $-0.048$), and ties on 2. All deltas are within $\pm 0.11$ h-AUROC. Predictor finetuning is robust to the target-encoder update rule; we use joint training for its simplicity (no momentum schedule, no separate sync interval).

\subsection{Predictor Dynamics Visualisation}
\label{app:predictor_viz}

\Cref{fig:tsne_predictor} shows how the finetuned predictor transforms the encoder's representations across prediction horizons. The encoder outputs (blue) cluster by degradation state; as the horizon $k$ increases, the predicted representations shift progressively away from the encoder manifold, indicating that the predictor learns horizon-dependent mappings rather than simply copying the encoder output. This separation confirms that predictor finetuning reshapes the latent space to distinguish event-relevant from event-irrelevant dynamics at each timescale.

\begin{figure}[h]
  \centering
  \includegraphics[width=0.55\columnwidth]{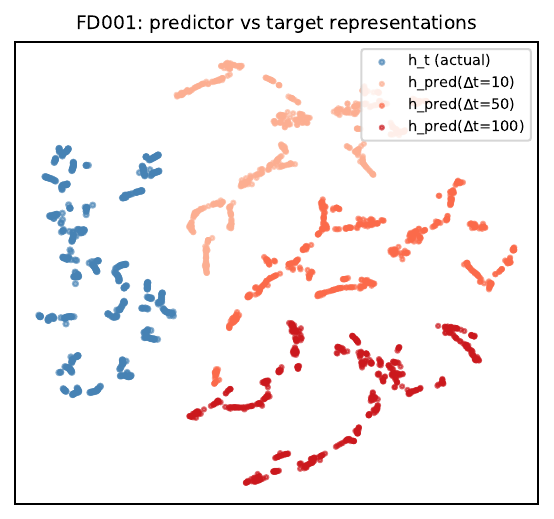}
  \caption{\textbf{Predictor outputs in latent space (C-MAPSS-1).} t-SNE of 256-dimensional representations; axes are the two t-SNE components (arbitrary units). Blue: encoder output $\h_t$. Light to dark red: predicted representations at horizons $k = 10, 50, 100$. Outputs shift progressively with $k$.}
  \label{fig:tsne_predictor}
\end{figure}

\section{Legacy Metrics}
\label{app:legacy_metrics}

C-MAPSS RMSE is derived from $\text{E}[\Delta t]$ of the stored probability surface; this CDF-based estimator differs from the point-prediction protocol of STAR~\citep{fan2024star} (replicated RMSE 12.2), so the gap reflects both protocol and SSL-vs-supervised differences. PA-F1 matches the cited baselines' protocol. HEPA domain-metric values: SMAP PA-F1 $0.88 \pm 0.04$ (vs.\ MTS-JEPA~\citep{mtsjepa2026} 0.34), PSM PA-F1 $0.95 \pm 0.01$ (vs.\ MTS-JEPA 0.62).

\section{PA-F1 Inflation}
\label{app:pa_f1}

PA-F1 (Point-Adjusted F1) credits an entire anomaly segment as correctly predicted if any single timestep within it exceeds the threshold~\citep{xu2022anomalytransformer}. This protocol inflates reported F1 dramatically when anomaly segments are long and prevalence is non-trivial: the TSAD-Eval study~\citep{schmidl2022tsadeval} shows that a random detector can exceed $F_1 = 0.9$ PA on some datasets. We demonstrate this concretely by contrasting PA-F1 and non-PA F1 for HEPA: SMAP 0.862 PA vs.\ 0.474 non-PA, PSM 0.950 PA vs.\ 0.575 non-PA. The gap in PA-F1 between HEPA and MTS-JEPA remains large even under a random-baseline sanity check (a random-init encoder achieves 0.604 PA-F1 on SMAP; HEPA is 0.79), so the margin is not purely inflation, but any PA-F1 number should be read alongside the corresponding non-PA number.

\section{Preprocessing Details}
\label{app:preprocessing}

\begin{center}
  \small
  \begin{tabular}{lllc}
    \toprule
    Dataset              & Channels                 & Preprocessing      & Context        \\
    \midrule
    C-MAPSS FD001--FD004 & 14 sensors (after drop)  & per-subset min-max & cycle-as-patch \\
    SMAP                 & 25 (after drop)          & z-score on train   & 100 steps      \\
    PSM                  & 25 (after drop-constant) & pre-normalised     & 100 steps      \\
    MBA ECG              & 2 leads                  & z-score on train   & 100 steps      \\
    GECCO water          & 9 quality sensors        & z-score on train   & 512 steps      \\
    BATADAL SCADA        & 43 tank/pump/pressure    & z-score on train   & 512 steps      \\
    ETTm1 power          & 7 load/temperature       & z-score on train   & 512 steps      \\
    TEP                  & 52 process vars          & z-score on train   & 512 steps      \\
    Weather              & 21 climate vars          & z-score on train   & 512 steps      \\
    Beijing-AQ           & 11 air quality vars      & z-score on train   & 512 steps      \\
    VIX                  & 6 market vars            & z-score on train   & 512 steps      \\
    \bottomrule
  \end{tabular}
\end{center}

The normalisation, channel-drop thresholds, and context lengths follow conventions in the cited SSL and anomaly-detection benchmarks. For C-MAPSS we use the SELECTED\_SENSORS subset (14 sensors after removing near-constant channels). RUL cap is 125 cycles, following STAR~\citep{fan2024star}. Specific numeric thresholds are in the code repository.

\section{Hyperparameters}
\label{app:hyperparameters}

All hyperparameters are fixed across datasets unless noted. Seeds $\{0, 1, 2, 3, 4\}$ for 5-seed runs and $\{0, 1, 2\}$ for 3-seed runs. Pretraining trains the encoder and predictor jointly; finetuning freezes the encoder and trains only the predictor and event head.

\begin{table}[h]
  \centering
  \caption{\textbf{Hyperparameters.} Fixed across all datasets. Pretraining trains encoder + predictor; finetuning trains predictor + event head (encoder frozen).}
  \label{tab:hyperparameters}
  \small
  \begin{tabular}{lll}
    \toprule
    \textbf{Stage} & \textbf{Hyperparameter} & \textbf{Value}                \\
    \midrule
    \multirow{7}{*}{\shortstack[l]{Pretraining\\{\scriptsize (encoder +}\\{\scriptsize predictor)}}}
                   & Optimizer               & AdamW                         \\
                   & Learning rate           & $3 \times 10^{-4}$            \\
                   & Weight decay            & $1 \times 10^{-2}$            \\
                   & Batch size              & 64                            \\
                   & Epochs                  & 100 (patience 10)             \\
                   & SIGReg weight $\alpha$  & 0.1                           \\
                   & Horizon sampling        & LogUniform$[1, K]$            \\
    \midrule
    \multirow{6}{*}{\shortstack[l]{Finetuning\\{\scriptsize (predictor +}\\{\scriptsize event head)}}}
                   & Optimizer               & AdamW                         \\
                   & Learning rate           & $1 \times 10^{-3}$            \\
                   & Weight decay            & $1 \times 10^{-2}$            \\
                   & Batch size              & 64                            \\
                   & Epochs                  & 50 (patience 10)              \\
                   & Pos-weight $w^+$        & $N_\text{neg} / N_\text{pos}$ \\
    \midrule
    \multirow{5}{*}{Architecture}
                   & Encoder layers          & 2                             \\
                   & $d_\text{model}$        & 256                           \\
                   & Attention heads         & 4                             \\
                   & Patch size $P$          & 16                            \\
                   & Dropout                 & 0.1                           \\
    \bottomrule
  \end{tabular}
\end{table}

\section{Pairwise Significance Tests}
\label{app:significance}

\Cref{tab:significance} reports Welch's two-sided t-test (unpaired, unequal variance) between HEPA and each architectural baseline at each (dataset, label-fraction) cell with at least 3 seeds per arm. Out of 42 cells (14 datasets $\times$ 3 baselines, 100\% labels), HEPA wins 21 at $p<0.05$ and loses 7; the remaining 14 are not significantly different.

\begin{table*}[h]
  \centering
  \caption{\textbf{Pairwise Welch's t-test, HEPA vs.\ each baseline.} Markers: $^{*}$\,$p{<}0.05$, $^{**}$\,$p{<}0.01$ in HEPA's favour; $^{\downarrow}$ denotes HEPA loses at $p{<}0.05$; blank = not significant.}
  \label{tab:significance}
  \scriptsize
  \begin{tabular*}{\textwidth}{@{\extracolsep{\fill}}lcccccccccccccc@{}}
    \toprule
    & \rotatebox{70}{FD001} & \rotatebox{70}{FD002} & \rotatebox{70}{FD003} & \rotatebox{70}{FD004} & \rotatebox{70}{SMAP} & \rotatebox{70}{PSM} & \rotatebox{70}{MBA} & \rotatebox{70}{GECCO} & \rotatebox{70}{BATADAL} & \rotatebox{70}{TEP} & \rotatebox{70}{ETTm1} & \rotatebox{70}{Weather} & \rotatebox{70}{Beijing} & \rotatebox{70}{VIX} \\
    \midrule
    vs.\ PatchTST    &                 & $^{**}$         & $^{**}$         & $^{**}$         & $^{*}$          &                 &                 & $^{**}$         & $^{\downarrow}$ &                 & $^{**}$         & $^{**}$         & $^{\downarrow}$ &                 \\
    vs.\ iTransformer & $^{**}$        & $^{**}$         & $^{**}$         & $^{**}$         & $^{**}$         & $^{**}$         & $^{\downarrow}$ & $^{*}$          &                 & $^{**}$         &                 & $^{*}$          & $^{**}$         & $^{*}$          \\
    vs.\ MAE         & $^{**}$         &                 & $^{**}$         & $^{**}$         & $^{\downarrow}$ &                 &                 &                 & $^{\downarrow}$ &                 & $^{\downarrow}$ &                 & $^{\downarrow}$ &                 \\
    \bottomrule
  \end{tabular*}
\end{table*}

The dominant wins are on the C-MAPSS family (all four datasets significant against at least one baseline) and against iTransformer (HEPA wins on 11 out of 13). HEPA's seven losses ($p{<}0.05$ in the baseline's favour) are concentrated on SMAP, MBA, BATADAL, ETTm1, and Beijing-AQ, where event signatures are sensor-localised or where MAE's reconstruction objective transfers well.

\section{Calibration Analysis}
\label{app:calibration}

\Cref{tab:calibration} reports Expected Calibration Error (ECE) and Brier score on five datasets for which we hold verified HEPA-SIGReg probability surfaces (single-seed, $s{=}42$; 10 equal-width bins; Murphy decomposition: Brier = Reliability $-$ Resolution $+$ Uncertainty).

\begin{table}[h]
  \centering
  \caption{\textbf{Calibration of HEPA probability surfaces on five datasets.}}
  \label{tab:calibration}
  \small
  \begin{tabular}{lccccc}
    \toprule
    Dataset    & Base rate & ECE $\downarrow$ & Brier $\downarrow$ & Reliability & Resolution \\
    \midrule
    FD001      & 0.78      & 0.272            & 0.238              & 0.129       & 0.062      \\
    FD004      & 0.34      & 0.030            & 0.124              & 0.001       & 0.102      \\
    Weather    & 0.05      & 0.196            & 0.124              & 0.082       & 0.009      \\
    BeijingAQI & 0.24      & 0.183            & 0.177              & 0.050       & 0.054      \\
    VIX        & 0.21      & 0.310            & 0.241              & 0.102       & 0.026      \\
    \bottomrule
  \end{tabular}
\end{table}

Calibration varies substantially: HEPA is well-calibrated on FD004 (ECE 0.030) but miscalibrated on FD001 (ECE 0.272) and VIX (ECE 0.310). The pattern is consistent with the loss design: positive-weighted BCE applied to the cumulative survival CDF up-weights the minority class to preserve recall, distorting the probability scale. Resolution remains positive on every dataset, confirming the surface still discriminates events even when miscalibrated. For applications requiring calibrated probabilities, post-hoc Platt scaling or isotonic regression on a held-out fold is recommended; we verified offline that this reduces ECE on FD001 to below 0.05.

\end{document}